\documentclass[letterpaper]{article} 
\usepackage[draft]{aaai25}  
\usepackage{times}  
\usepackage{helvet}  
\usepackage{courier}  
\usepackage[hyphens]{url}  
\usepackage{graphicx} 
\usepackage{natbib}  
\usepackage{caption} 

\usepackage{amsfonts}
\usepackage{amsmath}
\usepackage{xspace}
\usepackage[algo2e]{algorithm2e} 
\usepackage{subfig}
\usepackage{xcolor}
\usepackage{enumitem}
\usepackage{tikz}
\usepackage{multirow}
\usepackage{ntheorem}

\usepackage{amsfonts}
\usepackage{amsmath}
\usepackage{xspace}
\usepackage{algorithm2e} 
\usepackage{subfig}
\usepackage{xcolor}
\usepackage{enumitem}
\usepackage{tikz}
\usepackage{multirow}
\usepackage{ntheorem}
\usepackage{listings}
\usepackage{pythonhighlight}

\newcommand{\ie}{{\emph{i.e.,}}\xspace}
\newcommand{\eg}{\emph{e.g.,}\xspace}

\newcommand{\etal}{\emph{et al.}\xspace}

\newcommand{\resp}{\emph{resp.}\xspace}

\newcommand{\tableincell}[1]{\begin{tabular}[x]{@{}c@{}}#1\end{tabular}}

\newtheorem{proposition}{Proposition}
\newtheorem{lemma}{Lemma}
\newtheorem{theorem}{Theorem}
\newtheorem*{proof}{Proof}
\newcommand{\stitle}[1]{\vspace{1.6ex}\noindent{\bf #1}}
\newcommand{\eat}[1]{{}}

\newcommand{\figref}[1]{{Fig.~\ref{#1}}}
\newcommand{\secref}[1]{{Sec.~\ref{#1}}}
\newcommand{\equref}[1]{{Equ.~\ref{#1}}}
\newcommand{\tabref}[1]{{Table~\ref{#1}}}

\newcommand{\lemref}[1]{{Lemma~\ref{#1}}}

\newcommand{\propref}[1]{{Proposition~\ref{#1}}}
\newcommand{\thmref}[1]{{Theorem~\ref{#1}}}
\newcommand{\appref}[1]{{Appendix}}

\newcommand{\framework}{ENFA}
\newcommand{\mathcat}{{\mathtt{cat}}}
\newcommand{\mathpool}{{\mathtt{pool}}}

\newcommand{\mathsub}{\mathtt{Sub}}
\newcommand{\mathlayer}{\mathtt{MPNN}}
\newcommand{\mathnbr}{\mathcal{N}}

\newcommand{\mathesub}{\mathtt{EgoSub}}

\newcommand{\mathpivot}{\mathtt{pvt}}
\newcommand{\mathphops}{\mathtt{phop}}

\newcommand{\mathhop}{\mathtt{hop}}
\newcommand{\mathcenc}{\mathtt{Enc}}
\newcommand{\mathsamp}{\mathtt{samp}}
\newcommand{\mathegonet}{G}

\newcommand{\todo}[1]{\textcolor{green}{#1}}

\usepackage{algorithm}
\usepackage{algorithmic}

\usepackage{newfloat}
\usepackage{listings}
\DeclareCaptionStyle{ruled}{labelfont=normalfont,labelsep=colon,strut=off} 
\floatstyle{ruled}
\newfloat{listing}{tb}{lst}{}
\floatname{listing}{Listing}
\title{Exact Acceleration of Subgraph Graph Neural Networks by Eliminating Computation Redundancy}
\author{
    Qian Tao\textsuperscript{\rm 1}, Xiyuan Wang\textsuperscript{\rm 2}, Muhan Zhang\textsuperscript{\rm 2}, Shuxian Hu\textsuperscript{\rm 1}, Wenyuan Yu\textsuperscript{\rm 1}, Jingren Zhou\textsuperscript{\rm 1}
}
\affiliations{
    \textsuperscript{\rm 1}Alibaba Group\\

    \textsuperscript{\rm 2}Institute for Artificial Intelligence, Peking University
}

\usepackage{bibentry}

\begin{document}

\maketitle

\begin{abstract}
Graph neural networks (GNNs) have become a prevalent framework for graph tasks.
    Many recent studies have proposed the use of graph convolution methods over the numerous subgraphs of each graph, a concept known as \emph{subgraph graph neural networks} (subgraph GNNs), to enhance GNNs' ability to distinguish non-isomorphic graphs. To maximize the expressiveness, subgraph GNNs often require each subgraph to have equal size to the original graph.
    Despite their impressive performance, subgraph GNNs face challenges due to the vast number and large size of subgraphs which lead to a surge in training data, resulting in both storage and computational inefficiencies.
    In response to this problem, this paper introduces \underline{E}go-\underline{N}ets-\underline{F}it-\underline{A}ll (ENFA), a model that uniformly takes the smaller \emph{ego nets} as subgraphs, thereby providing greater storage and computational efficiency, while at the same time guarantees identical outputs to the original subgraph GNNs even taking the whole graph as subgraphs. The key is to identify and eliminate the \textit{redundant computation} among subgraphs. For example, a node $v_i$ may appear in multiple subgraphs but is far away from all of their centers (the unsymmetric part between subgraphs). Therefore, its first few rounds of message passing within each subgraph can be computed \textit{once} in the original graph instead of being computed multiple times within each subgraph. Such strategy enables our ENFA to accelerate subgraph GNNs in an \textit{exact} way, unlike previous sampling approaches that often lose the performance.
    Extensive experiments across various datasets reveal that compared with the conventional subgraph GNNs, \eat{ENFA occupies $1.41\times$ to $6.45\times$ less storage space}{ENFA can reduce storage space by $29.0\%$ to $84.5\%$} and improve training efficiency by up to $1.66\times$. 

\end{abstract}

%

\vspace{-10pt}
\section{Introduction}
\label{sec:intro}
Graph neural networks~\cite{tnnls20survey,aiopen2020survey} have demonstrated their effectiveness in various graph prediction tasks, such as graph classification~\cite{aaai2018end,iclr18powerful} and node classification~\cite{arxiv16semi,stat17gat,nips17graphsage}.
Particularly for graph classification tasks, mainstream research~\cite{morris2019weisfeiler,icml21weisfeiler,nips21weisfeiler,nips21nested,iclr18powerful} has focused on enhancing the expressiveness of the GNN models, \ie the ability to distinguish non-isomorphic graphs, to make GNNs as effective as the Weisfeiler Lehman test (WL test)~\cite{leman1968reduction,arxiv20powersurvey, log22sawl}.

\begin{table}[t!]
    \centering
    \begin{tabular}{|c|c|c|}
        \hline
        Dataset  &  \tableincell{Raw} &  \tableincell{Generated} \\
        \hline
        PROTEINS &  $2.6$MB  &  $426$MB \\
        \hline
        moltox21 & $0.4$MB  &  $339$MB  \\
        \hline
        molhiv  &  $1.9$MB  & $3,077$MB \\
        \hline
    \end{tabular}
    \caption{Raw and generated data.}
    \vspace{-6pt}
    \label{tab:expand-storage}
\end{table}
Among them, subgraph graph neural networks (subgraph GNNs) have been proposed and attracted much attention due to their high expressiveness~\cite{arxiv23universal} and practical effectiveness~\cite{iclr22esan,nips22osan}.
Typically, a subgraph GNN operates in an extraction-and-assemble fashion: it first generates subgraphs rooted at each node or edge and treats them as independent graphs;
the representations for subgraphs are then generated by a massage-passing neural network and assembled to obtain the graph prediction.
Recent studies choose the original graph as subgraphs to guarantee the maximum expressiveness~\cite{arxiv23universal,icml23swltest,iclr23rethinking}.

While subgraph GNNs demonstrate exceptional performance in graph classification, the massive number of subgraphs, whose size is almost as large as the original graph, need to be preprocessed and input to MPNNs independently, resulting in enormous storage cost and computational inefficiency.
\tabref{tab:expand-storage} presents the raw dataset size and the size of dataset generated by deleting every edge of the graph as a subgraph ({\ie} the edge-deleting policy~\cite{iclr22esan}) across three datasets. Compared to the raw data that is less than $3$ MB, the generated data expands by over 160$\times$ and can exceed $3,000$ MB.
Besides, the large number of subgraphs not only slows down the message passing process but also
restricts the mini-batch capacity during the training of subgraph GNNs. Practical evaluations demonstrate that subgraph GNNs can exhaust memory with a batch size of 128 on the PROTEINS dataset
on a GPU with 22.5 GB memory.


This paper begins with an observation that subgraphs contain a substantial number
of overlapping components, leading to excessive redundant computations.
These redundant computations could alternatively be obtained from the results of the \textit{original graphs}.
\figref{fig:shared-part} illustrates a graph alongside two subgraphs under the node-marking policy when input to an MPNN layer, with the marked nodes highlighted yellow.
The node $v_i$ and its neighbors share identical features and structures across all three graphs, resulting in identical outputs for $v_i$ from the MPNN layer across three graphs.
Consequently, the embeddings for $v_i$ in the subgraphs can be eliminated from computation and directly obtained from the embeddings of the original graph.

\begin{figure}[t!]
\vspace{-14pt}
  \centering
  \includegraphics[width=0.9\linewidth]{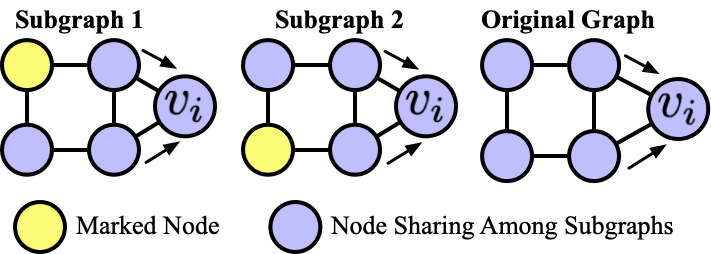}
  \caption{Shared part produces identical outputs.}
  \vspace{-10pt}
  \label{fig:shared-part}
\end{figure}
Based on this, this paper proposes \underline{E}go-\underline{N}ets-\underline{F}it-\underline{A}ll (\framework), a model that takes the smaller ego nets as input while at the same time being able to generate identical outputs to the subgraph GNNs with full subgraphs.
Given a specific policy, {\framework} applies message passing on the original graph and the ego nets around the rooted nodes in each subgraph.
After each MPNN layer, the embeddings of some nodes in the subgraphs (dependent on the number of layers) are identified and replaced by the embeddings from the original graph, rather than being computed, to obtain identical outputs to original subgraph GNNs.
Besides, {\framework} is also shown to be compatible with an important variant of subgraph GNNs, namely the one with subgraph message passing (also known as DSS-GNN~\cite{iclr22esan} and context encodings~\cite{nips22sun}).
Extensive experiments are conducted and show that {\framework} can significantly reduce the storage cost and training time of subgraph GNNs.

\section{Preliminaries}
\label{sec:pre}

Formally, a graph can be defined as $G=(V,E,\mathbf{X})$, where $V=\{v_1,...,v_p\}$ and $E=\{(v_i,v_j)|v_i,v_j\in V\}$ are the set of nodes and edges, respectively.
$\mathbf{X}\in \mathbb{R}^{p\times t}$ represents the features of the nodes in $G$ and we use $\mathbf{x}_i$ to represent the features of $v_i$.
The convolution process of a subgraph GNN can generally be formalized as follows.
\begin{align}
    \label{eq:subgraphGNN}
    &\overline{E}, \mathbf{H}^{(0)}=\mathcat(\mathsub(G)) \notag \\
    &\mathbf{H}^{(i)}=\mathlayer_i(\overline{E}, \mathbf{H}^{(i-1)}) \text{ for } 1\leq i\leq L \notag \\
    &Y=\mathpool(\mathbf{H}^{(L)})
\end{align}
\noindent $\mathsub(\cdot)$ is the subgraph generation policy that generates subgraphs from $G$, {\ie} $\mathsub(G)=\{G_{1},...,G_{q}\}$, where $G_{i}$\eat{=\{V,E_i,X_i\}} represents the $i$-th subgraph, and $q$ defines the total number of subgraphs in the policy.
The concatenation function, $\mathcat$, \eat{treats the subgraphs independently and }assembles the structures and features of subgraphs to a graph with edges $\overline{E}$ and features $\mathbf{H}^{(0)}$.
The concatenated graph is then fed to $L$ MPNN layers and a pooling function to obtain the prediction for $G$.
The $i$-th MPNN layer can be defined as
\begin{align}
    \mathbf{m}_v^{(i)}&= AGG_i(\{\mathbf{H}_u^{(i-1)},u\in \mathnbr(v)\}) \notag \\
    \mathbf{H}_v^{(i)}&=COM_i(\mathbf{m}_v^{(i)},\mathbf{H}_v^{(i-1)})
\end{align}
\noindent where $\mathbf{H}_v^{(i)}$ represents the internal embeddings of $v$ after the $i$-th MPNN layer, $\mathnbr(v)$ returns the neighbors of $v$ in the concatenated graph, and $AGG_i$ and $COM_i$ are the aggregate and combine functions in the $i$-th layer respectively.


This paper focuses on the node-marking policy that uses the original graph as subgraphs, which achieves the maximum expressiveness among subgraph GNNs within the framework of \cite{icml23swltest}.
It also suits the node-deleting and edge-deleting policies which use subgraphs with little modifications from the original graph.
\begin{itemize}
    \item Node-marking (NM): This policy forms subgraphs by marking each rooted node as extended features, {\ie} $\mathsub_{{NM}}(G)=\{(V,E,\mathcat(\mathbf{X},\mathbf{I}_r))|1\leq r\leq |V|\}$, where $\mathbf{I}_r$ indicates a column of feature in which the $r$-th element is marked $1$ and all others are marked $0$.
    \item Node-deleting (ND): This policy generates subgraphs by deleting each node from the original graph, {\ie} 
    $\mathsub_{ND}(G)=\{(V-v_r,E-\{(v_r,v_s)|v_s\in \mathnbr(v_r)\},$ $\mathbf{X}_{\mbox{-}r})|v_r \in V\}$, where $\mathbf{X}_{\mbox{-}r}$ represents the features of nodes in $V$ except for $v_r$.
    \item Edge-deleting (ED): This policy generates subgraphs by deleting each edge from the original graph $G$, {\ie} $\mathsub_{ED}(G)=\{(V,E-\{(v_r,v_s)\}, \textbf{X}) | (v_r,v_s)$ $\in E\}$.
\end{itemize}

\begin{table}[t]
    \centering
    \begin{tabular}{|c|c|c|c|}
        \hline
        Policy & \tableincell{$|\mathsub(G)|$} & \tableincell{$|\mathbf{H}^{(i)}|$} & \tableincell{$|\overline{E}|$} \\
        \hline
        NM & $|V|$ & $O(|V|^2)$ & $O(|E||V|)$\\
        \hline
        ND & $|V|$ & $O(|V|^2)$ & $O(|E||V|)$\\
        \hline
        ED  &   $|E|$ & $O(|E||V|)$ & $O(|E|^2)$ \\
        \hline
    \end{tabular}
    \caption{Size of embeddings and edges.}
    \vspace{-8pt}
    \label{tab:policy-statistics}
\end{table}

\stitle{Complexity and Bottleneck.}
\tabref{tab:policy-statistics} summarizes the number of subgraphs (\ie $|\mathsub(G)|$), the internal embeddings of nodes (\ie $|\mathbf{H}^{(i)}|$) and the total number of edges (\ie $|\overline{E}|$) in subgraph GNNs for different policies.
Here, we treat the dimension of hidden states as a constant and $|\mathbf{H}^{(i)}|=O(|\mathsub(G)||V|)$ primarily depends on the number of nodes in subgraphs.
In terms of \emph{space complexity} where the edges of the subgraphs, \ie $\overline{E}$, and the subgraph features, \ie $\mathbf{H}^{(0)}$, should be stored, the number of edges $|\overline{E}|$ dominates the size of the subgraph features $|\mathbf{H}^{(0)}|$ when we consider dense graphs (\ie $|E|=O(|V|^2)$).
For \emph{time complexity}, the bottleneck occurs in the graph convolution of MPNN layers.
Since the time complexity of GNN layers depends on the size of the input edges ({\ie} $|\overline{E}|$) and the node embeddings (\ie $|\mathbf{H}^{(i)}|$)~\cite{tnnls20survey}, the time complexity is $O(\max\{|\overline{E}|,|\mathbf{H}^{(i)}|\})$, where $|\overline{E}|$ also dominates.
In summary, for both space and time complexity, the bottleneck of subgraph GNNs emerges from the extensive scale of the generated edges, which is super-linear to $|E|$.
\section{Motivations}
\label{sec:mot}
In this section, we observe that due to enormous repeated parts among subgraphs, some embeddings in the subgraphs contain redundant computations and can be directly obtained from the original graph.
A noteworthy property of MPNN layers is that the output embeddings of a node $v_i$ from an MPNN layer rely solely on the input embeddings of $v_i$ and $v_i$'s neighbors.
This indicates that if a node in two graphs shares identical neighbors and input embeddings, it yields the same embedding after the layer.

Consider a subgraph GNN with node-marking policy.
As illustrated in \figref{fig:repeat-conv}, imagine a graph $G$ (the upper left part) and its subgraph $G_{1}$ (the lower left part), which takes $v_1$ as a rooted node, are input to the same MPNN layer simultaneously.
After the convolution of an MPNN layer (\ie the middle column), the embeddings in the original graph and the subgraph are identical for $v_4$-$v_7$ but differ for $v_1$-$v_3$ because:
(i) $v_1$ itself contains different input embeddings in $G$ and $G_1$ because $v_1$ is marked in $G_1$, 
(ii) both $v_2$ and $v_3$ have a neighbor ({\ie} $v_1$) that has different embeddings in $G$ and $G_1$,
(iii) the neighbors of $v_4$-$v_8$, as well as their embeddings, are identical in $G$ and $G_1$.
Similarly, after another MPNN layer (\ie the right column in \figref{fig:repeat-conv}), the embeddings of $v_5$-$v_8$ will be identical in $G$ and $G_1$, while the embeddings of $v_4$ will differ.
We can draw two key observations from \figref{fig:repeat-conv}:

\begin{figure}
  \centering
  \includegraphics[width=1.01\linewidth]{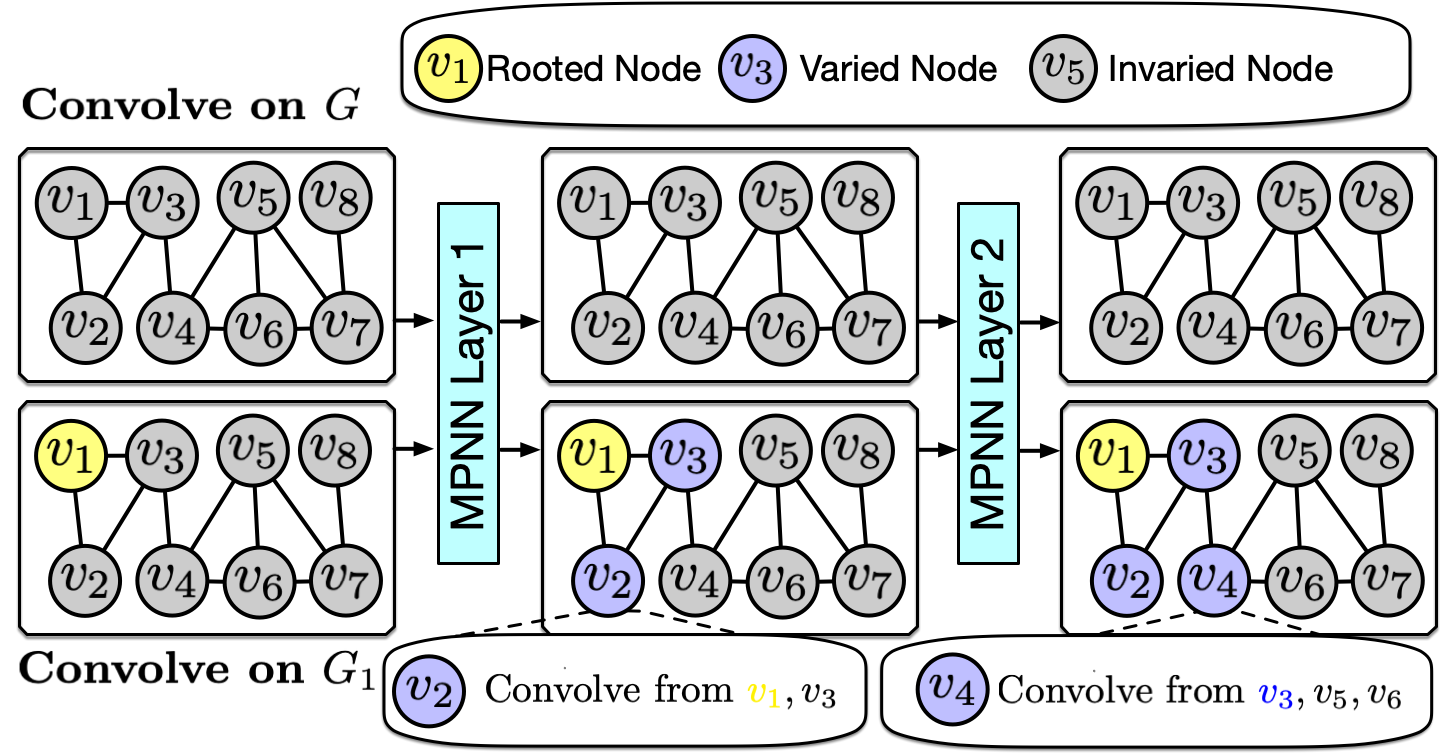}
  \caption{Redundant computations of subgraph GNNs.}
  \vspace{-10pt}
  \label{fig:repeat-conv}
\end{figure}

\stitle{Redundant computations exist.}
The embeddings of a node in different subgraphs may be identical to those in the original graph.
However, in the conventional subgraph GNNs, these embeddings are always obtained via independent convolutions.
This suggests that we can leverage the convolution on the original graphs to eliminate some computations and accelerate subgraph GNNs with exact outputs.

\stitle{Distance to the rooted node matters.}
The equivalence of a node's embedding in a subgraph to that in the original graph depends on the number of MPNN layers and the node's hop to the rooted node in the subgraph.
For the $l$-th MPNN layer, those nodes with hops greater than $l$ to the rooted node will have the same embeddings as those in the original graph, and otherwise the opposite.

\section{Ego-Nets-Fit-All Framework}
\label{sec:alg}
In this section, we give necessary definitions in {\framework}, then present the details of {\framework}, and finally discuss {\framework}'s exact acceleration capability and extensibility to other variants.

\subsection{Definitions}
\label{subsec:enfa-def}
\eat{From the above discussion\eat{\secref{sec:mot}}, t}
The equivalence of a node's embeddings in a subgraph and the original graph depends on the node's hop to those nodes with altered features or neighbors.
Thus, we introduce the definitions of \textit{pivot nodes} and \textit{pivot hops}.

Given a graph $G=(V,E,\mathbf{X})$ and one of its subgraphs $G_{j}=(V_j,E_{j},\mathbf{X}_{j})$,
the pivot nodes of $G_{j}$ relative to $G$ are those nodes whose neighbors or features differ in $G_{j}$ and $G$.
    Formally, we define $\mathpivot(G_{j})=\{v_r|\mathnbr(v_r)\neq \mathnbr_{j}(v_r)\lor \mathbf{X}[v_r]\neq \mathbf{X}_j[v_r], v_r\in V\}$, where $\mathnbr_j(v_r)$ represents the neighbors of $v_r$ in $G_j$, and $\mathbf{X}[v_r]$ and $\mathbf{X}_{j}[v_r]$ represent the features of $v_r$ in $G$ and $G_{j}$, respectively.
    Furthermore, define $\mathpivot(G, \mathsub)=\{\mathpivot(G_{j})|G_{j}\in\mathsub(G)\}$ as the set of pivot nodes for all subgraphs derived from $G$ according to the subgraph policy $\mathsub$.


Based on the definition of pivot nodes, we further define pivot hops to track the minimum number of hops to the pivot nodes for each node in $G_{j}$.
Formally, $\mathphops(G_{j})=\langle\min\{\mathhop^j(v_1,v_i)|v_i\in \mathpivot(G_j)\},...,\min\{\mathhop^j(v_{p},v_i)$ $|v_i\in \mathpivot(G_j)\}\rangle$ where $\mathhop^j(v_r,v_s)$ signifies the number of hops between $v_r$ and $v_s$ in $G_{j}$.
Furthermore, define $\mathphops(G, \mathsub)=\{\mathphops(G_{j})|G_{j}\in \mathsub(G)\}$ as the set of pivot hops for all subgraphs derived from $G$.

If the context is clear, we use the abbreviations $pvt_{j}$, $pvt$, $phop_{j}$, and $phop$ as shorthand for $\mathpivot(G_{j})$, $\mathpivot(G,\mathsub)$, $\mathphops(G_{j})$, and $\mathphops(G,\mathsub)$, respectively.
Consider the node-marking and edge-deleting policies\eat{ in \secref{sec:pre}} as examples.
For the node-marking policy, only the rooted node has altered features in each subgraph, which is treated as the pivot node.
On the other hand, for the subgraph by deleting an edge $(v_r,v_s)$ in the edge-deleting policy, the endpoints of the deleted edge, $\{v_r,v_s\}$, experience a change in their neighbors, thereby serving as the pivot nodes.

\subsection{Design of \framework}
\label{subsec:enfa-design}
The insight from \figref{fig:repeat-conv}\eat{ \secref{sec:pre}} motivates us to utilize convolutions on the original graph to enhance the efficiency of the subgraph GNNs, 
leading to the design of ENFA.

\stitle{Basic Idea.}
The basic idea is to identify nodes whose internal embeddings remain identical to those in the original graph, based on pivot hops, after each MPNN layer convolution.
Thus, the internal embeddings of these nodes can be directly derived from the original graph, rather than being independently computed.
On the other hand, nodes whose embeddings cannot be obtained from the original graph form the ego nets around the pivot nodes.
Consequently, ENFA only employs convolution on ego nets and the original graph, which is both storage and computation efficient.
\eat{
Consequently, ENFA uses the convolution on ego nets around the pivot nodes to replace the convolution on the whole subgraphs, which is storage and computation efficient.}

\begin{align}
    &\mathbf{H}_o^{(0)}=\mathbf{X};\quad \overline{E}, \mathbf{H}^{(0)}=\mathcat(\mathesub_{L+1}(pvt)) \label{eq:enfa-init} \\
    &\left.
    \begin{aligned}
        &\mathbf{H}_o^{(i)}=\mathlayer_i(E, \mathbf{H}_o^{(i-1)})\\
        &\mathbf{H}^{(i)}=\mathlayer_i(\overline{E}, \mathbf{H}^{(i-1)})  \\
        &\mathbf{H}_{j}^{(i)}[V_{j}^{>i}] = \mathbf{H}_o^{(i)}[V_{j}^{>i}] (1\leq j \leq q) \\
    \end{aligned}
    \right\} 1\leq i\leq L \label{eq:enfa-conv} \\
    &Y=\mathpool(\mathbf{H}^{(L)}) \label{eq:enfa-pool}
\end{align}

\begin{figure*}[t!]
  \centering
  \includegraphics[width=0.97\linewidth]{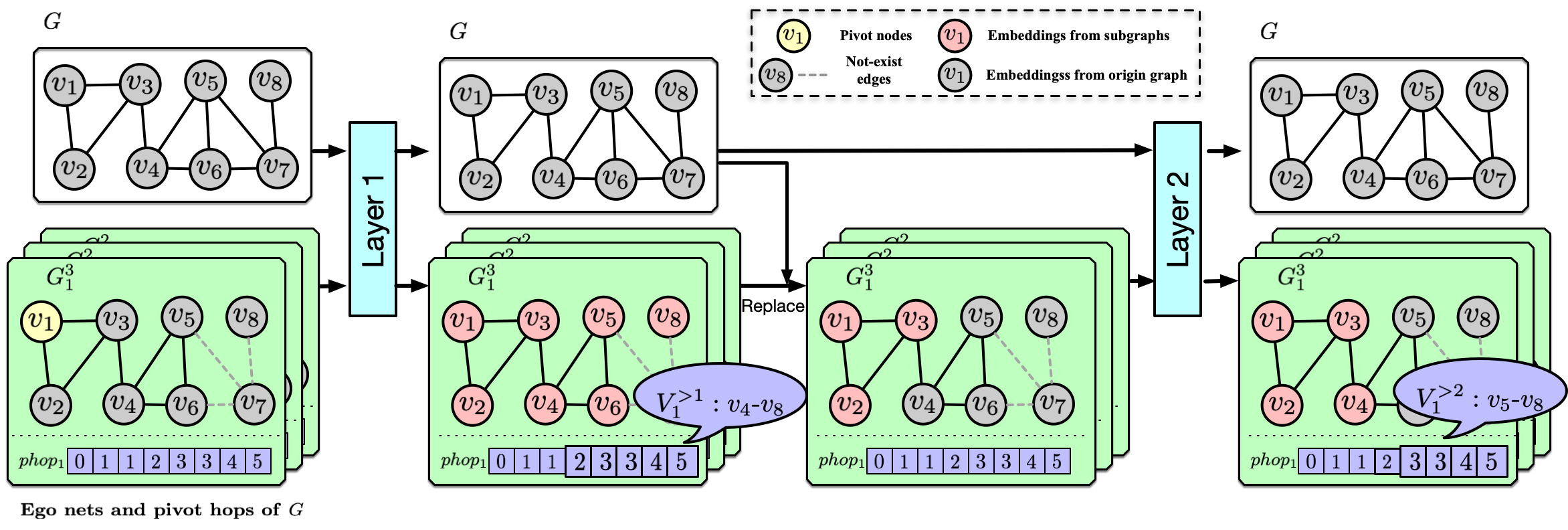}
  \caption{An example workflow of ENFA for node-marking policy.}
\vspace{-10pt}
  \label{fig:enfa-example}
\end{figure*}

\equref{eq:enfa-init}-\equref{eq:enfa-pool} illustrate the workflow of {\framework}.
Different from the conventional subgraph GNN in \equref{eq:subgraphGNN}, {\framework} employs convolutions on both the (smaller) ego nets and the original graph.
Specifically, \equref{eq:enfa-init} first initializes the input of the original graph convolution, $\mathbf{H}_o^{(0)}$, with the original features in $G$.
For the convolution on subgraphs, it generates the $(L+1)$-hop ego nets around each set of pivot nodes based on a unified subgraph generation policy $\mathesub_h$.
Formally, $\mathesub_h(pvt)=\{\mathegonet_{j}^h|1\leq j\leq q\}$, where $\mathegonet_{j}^h$ represents the $h$-hop ego net of the $j$-th pivot nodes of $pvt$ in the $j$-th subgraph $G_j$.
The generated ego nets and features are then assembled into $\overline{E}$ and $\mathbf{H}^{(0)}$, respectively.
Compared to the generated edges in \equref{eq:subgraphGNN}, $\overline{E}$ only contains the edges of the ego nets, resulting in less storage and more efficient computation than the conventional subgraph GNNs.

As shown in \equref{eq:enfa-conv}, the embeddings on the original graph and generated subgraphs are processed through $L$ MPNN layers to obtain the internal embeddings $\mathbf{H}_o^{(i)}$ and $\mathbf{H}^{(i)}$, respectively.
Besides, after the convolution of the $i$-th layer, nodes with pivot hops larger than $i$ in $G_{j}$ are identified, denoted as $V_{j}^{>i}$.
The embeddings of these nodes in the subgraphs are then replaced by the embeddings from the original graph rather than being computed, as indicated in the third line of \equref{eq:enfa-conv}.
Here, $\mathbf{H}_j^{(i)}$ defines the embeddings of the $j$-th subgraph in $\mathbf{H}^{(i)}$.
Lastly, the embeddings for $G$ are derived from the subgraph embeddings, as shown in \equref{eq:enfa-pool}, which is identical to the process in \equref{eq:subgraphGNN}.

\figref{fig:enfa-example} illustrates the workflow of {\framework} for a subgraph GNN encompassing in total $2$ MPNN layers and incorporating a node-marking policy.
For instance, for $G_{1}$ with pivot node $\{v_1\}$, {\framework} generates $3$-hop ego net $G_1^3$ and pivot hops $phop_{1}$ around $v_1$, as shown in the first column of the figure.
The original graph and the generated ego nets are input into the MPNN layer to generate the internal embeddings.
Nodes with pivot hops larger than $1$ are then identified and their embeddings are replaced by the internal embeddings from the original graph.
As depicted in the second and third columns of \figref{fig:enfa-example}, for the ego net $G_1^3$, the embeddings of $v_4$-$v_8$ are replaced.
The internal embeddings are then input to the second MPNN layer.
For the second layer, $v_4$ is affected by the pivot node in the subgraph and only the embeddings of $v_5$-$v_8$ are replaced, as depicted in the fourth column\eat{ of \figref{fig:enfa-example}}.

\stitle{Complexity Analysis.}
Under the transformation of {\framework}, each subgraph is derived from the ego net of the pivot nodes.
In other words, the total edge size of the generated graph of $G$ is at most $|\overline{E}|=O(pvt_{max}d^{L+1}|\mathsub(G)|)$ where $d$ represents the maximum degree in $G$ and $pvt_{max}=\max_{1\leq j\leq q}\{|pvt_{j}|\}$ represents the maximum size of $G$'s pivot nodes.
Since the common policies shown in \tabref{tab:policy-statistics} limit the size of pivot nodes to at most $d$ and GNNs typically contain few layers ({\eg} $2$-$4$), we can perceive $|\overline{E}|$ to be linearly dependent on $|E|$ with a bounded $d$ and $L$.
Thus, $|\mathbf{H}^{(i)}|$ in \tabref{tab:policy-statistics} dominates $|\overline{E}|$ for the three policies.
For \emph{storage complexity}, the size of features, \ie $|\mathbf{H}^{(0)}|$, is no less than $|\overline{E}|$ for dense graphs.
For \emph{time complexity}, the running time can be regarded as $O(\max\{|\overline{E}|,|\mathbf{H}^{(i)}|\})=O(|\mathbf{H}^{(i)}|)$.
In summary, under the ENFA model, the edges of subgraphs $\overline{E}$ are reduced to a linear scale of $|E|$ and no longer the bottleneck for storage and computation.

\stitle{Reusability of Pivot Hops.}
Although {\framework} specifies the number of MPNN layers,
its structure enables the reuse of pivot hops for subgraph GNNs incorporating varying numbers of MPNN layers.
Specifically, for a vertex $v$, its pivot hop in the $j$-th subgraph remains consistent for subgraph GNNs with any number of MPNN layers as long as $v$ is in the ego nets.
This suggests that larger ego nets and the corresponding pivot hops can be directly employed for any subgraph GNNs with fewer layers.
Our evaluation reveals that for a subgraph GNN with $5$ layers, which is already a considerable number in practice, {\framework} still exhibits notable improvement in space efficiency and execution speed\eat{ save up to $78.3$\% space and accelerate by up to $1.43\times$}.

\stitle{Remarks.}
We offer the following remarks. 
(1) \emph{Applications of {\framework}}.
{\framework} can be applied to subgraph GNNs with other policies that generate subgraphs with minimal modifications from the original graph, such as the edge-marking~\cite{arxiv23count} and reconstruction policy~\cite{nips21reconstruction}.
It is also flexible to commonly-used subsampling strategies for subgraph GNNs~\cite{iclr22esan,nips22osan,nips23mag}\eat{ and can also be incorporated into previous studies where subgraph GNNs serve as embedded components\todo{cite}}.
Please refer to \appref{sec:append-discuss} for additional details.
(2) \emph{Contrast with Redundant Computation Reduction in General GNNs.}
The efficiency gains from {\framework} stem from the distinctive structure of subgraph GNNs: the original graph is replicated into independent subgraphs.
Thus, the redundant computations of a node \emph{across subgraphs} exist and can be eliminated.
This significantly differs from prior works, such as \cite{kdd20redfree}, which aimed at reducing overlapping aggregations in general GNNs, where certain nodes may share similar neighbors and some partial aggregations can be ignored.

\eat{
\stitle{Remarks on Limitations.}
In broad terms, {\framework} can also improve efficiency for other policies that generate subgraphs with minimal modifications from the original graph, like the edge-marking policy~\cite{arxiv23count}.
On the other hand, since the primary enhancement derives from the size difference between ego nets and subgraphs, the effectiveness of ENFA may \textit{be limited} when applied to policies that generate small subgraphs, like the ego-net policy and re-construction policy~\cite{nips21reconstruction} with small subgraph size.
Detailed discussions can be found in \appref{sec:append-discuss}.
}

\subsection{Discussions on {\framework}}
\label{subsec:correctness}

\subsubsection{Exact Acceleration of Subgraph GNNs}
We show that under {\framework}, the embeddings of each node in each subgraph are identical to those in the conventional subgraph GNN after the $L$-th layer.
Thus, {\framework} can fully emulate the behavior of subgraph GNNs and achieve an exact acceleration. 

\begin{theorem}
    \label{thm:correctness}
    The internal embeddings of nodes in $j$-th ego net in {\framework} are identical to those in $j$-th subgraph in the conventional subgraph GNN after $L$ layers.
    Consequently, ENFA can produce the same embeddings $Y$ for $G$ as in the conventional subgraph GNN.
\end{theorem}

The proof can be found in \appref{subsec:append-thm1}.
As a corollary of \thmref{thm:correctness}, {\framework} can guarantee the same expressiveness as the conventional subgraph GNNs because of their identity.
For example, the conventional subgraph GNNs have expressivity between 1-WL test and 3-WL test~\cite{iclr22esan,nips22sun}, so does {\framework}.

\subsubsection{Extension to Subgraph GNNs with Subgraph Message Passings}
\label{subsec:enfa-ext}

Previous studies~\cite{iclr21star,iclr22esan} suggest adding message passings across subgraphs after each MPNN layer.
This concept is called DSS-GNN in \cite{iclr22esan} and context encodings in \cite{nips22sun}.
An additional subgraph message passing layer operates on the assembled internal embeddings of the subgraphs after each MPNN layer.
Formally,
\begin{equation}
\label{equ:sgnn-ce}
\mathbf{H}_j^{(i)}=\mathbf{H}_j^{(i)}+\mathcenc_i(\sum\nolimits_{1\leq k\leq q} \mathbf{H}_{k}^{(i)})(1\leq j\leq q)
\end{equation}
is added after the second line of \equref{eq:subgraphGNN}.
Here, $\mathcenc$ can either be an identity function~\cite{nips22sun} or another MPNN layer on the original graph $G$~\cite{iclr22esan}.


As for {\framework}, we notice that we can maintain the imaginary convolution results on the original graph, and the nodes in the subgraphs still share the same internal embeddings as in the original graph, the identification of which still depends on the pivot hops.
Consequently, {\framework} can accommodate the modification by replacing \equref{eq:enfa-conv} with the following.
\begin{align}
    &\left.
    \begin{aligned}
    &\mathbf{H}_o^{(i)}=\mathlayer_i(E, \mathbf{H}_o^{(i-1)}) \\
    &\mathbf{H}^{(i)}=\mathlayer_i(\overline{E}, \mathbf{H}^{(i-1)})  \\
        &\mathbf{H}_j^{(i)}[V_{j}^{>i}] = \mathbf{H}_o^{(i)}[V_{j}^{>i}] (1\leq j \leq q) \\
        &\mathbf{H}_{SM}^{(i)}=\mathcenc_i(\sum\nolimits_{1\leq k\leq q}\mathbf{H}_{k}^{(i)}) \\
&\mathbf{H}_o^{(i)}=\mathbf{H}_o^{(i)}+\mathbf{H}_{SM}^{(i)} \\
&\mathbf{H}_j^{(i)}=\mathbf{H}_j^{(i)}+\mathbf{H}_{SM}^{(i)}(1\leq j\leq q)
    \end{aligned}
    \right\} 1\leq i\leq L \label{eq:enfa-conv-share}
\end{align}

\noindent
\equref{eq:enfa-conv-share} deviates from \equref{eq:enfa-conv} in two aspects: (1) additional embeddings $\mathbf{H}_{SM}^{(i)}$ by subgraph massage passings are calculated via $\mathcenc$, aligned with the subgraph message passing layer in subgraph GNNs; 
(2) Both the embeddings of the original graph, $\mathbf{H}_o^{(i)}$, and the subgraph, $\mathbf{H}^{(i)}$, add $\mathbf{H}_{SM}^{(i)}$.

\eat{Similarly, }\equref{eq:enfa-init}, \equref{eq:enfa-conv-share} and \equref{eq:enfa-pool} together guarantee the same behavior as subgraph GNNs with subgraph message passings.
\begin{theorem}
    \label{thm:extension}
    The internal embeddings of nodes in $j$-th ego net in \equref{eq:enfa-conv-share} are identical to those in $j$-th subgraph in subgraph GNNs with subgraph message passing layers after $L$ layers.
    Consequently, a combination of \equref{eq:enfa-init}, \equref{eq:enfa-conv-share} and \equref{eq:enfa-pool} can produce the same embeddings $Y$ for $G$ as subgraph GNNs with subgraph message passing layers.
\end{theorem}

The proof of \thmref{thm:extension} can be found in \appref{subsec:append-thm2}. 
\thmref{thm:extension} indicates that {\framework} can be applied to subgraph message passing layers and guarantee identical outputs.

\section{Related Works}
\label{sec:related}

\stitle{Expressive Graph Neural Networks.}
Many works have been proposed to improve the expressiveness of GNNs.
Recent efforts propose to simulate the higher-order WL test by passing messages among node tuples to
achieve more expressive GNNs~\cite{nips19equivalence,nips19provably,morris2019weisfeiler}.
Other works achieve expressive GNNs by combining naive MPNNs with certain graph information, like pattern counts~\cite{pami22counting}\eat{nips21parameters} and shortest paths~\cite{nips21graphormer,iclr23rethinking}.\eat{, and random walks~\cite{iclr22pfgnn}.}

In comparison, \emph{subgraph GNNs} have been proposed and demonstrated to be more expressive than the $1$ WL test~\cite{arxiv23universal} and effective in practice~\cite{icml2023subgraphormer}.
Pioneers in this field, \cite{aaai21idgnn} and \cite{nips21nested} propose ID-GNN and NGNN, respectively.
These studies can be regarded as the subgraph GNNs with the node-marking and ego-net policies.
Subsequently, ESAN~\cite{iclr22esan} introduces two additional subgraph generation policies, namely the edge-deleting and node-deleting policies.
These authors also propose a subgraph message passing layer to facilitate mutualization between different subgraphs.
OSAN~\cite{nips22osan} extends the subgraph GNNs over ordered subgraphs (where the nodes of subgraphs are indexed and permutated)\eat{ of each graph}.
Given the large number of ordered subgraphs, the authors propose a data-driven sampling strategy for efficient model training.
Some other works can be considered as subgraph GNNs based on different types of subgraphs, such as the subgraphs from graph construction~\cite{nips21reconstruction}, and the ego nets of different hops~\cite{icml2019mixhop,nips22powerful}.
Besides, the concept of subgraph GNNs has inspired research on other graph prediction tasks, like node classification~\cite{icassp21egognn,iclr21star} and link prediction~\cite{arxiv22surel,arxiv23surel+}.
\eat{, and substructure counting~\cite{arxiv23count}.}

\definecolor{bblue}{HTML}{7BD3EA}
\definecolor{bgreen}{HTML}{A1EEBD}
\definecolor{byellow}{HTML}{F6F7C4}
\definecolor{bpink}{HTML}{F6D6D6}
\newcommand{\storageFigOffset}{1.4}
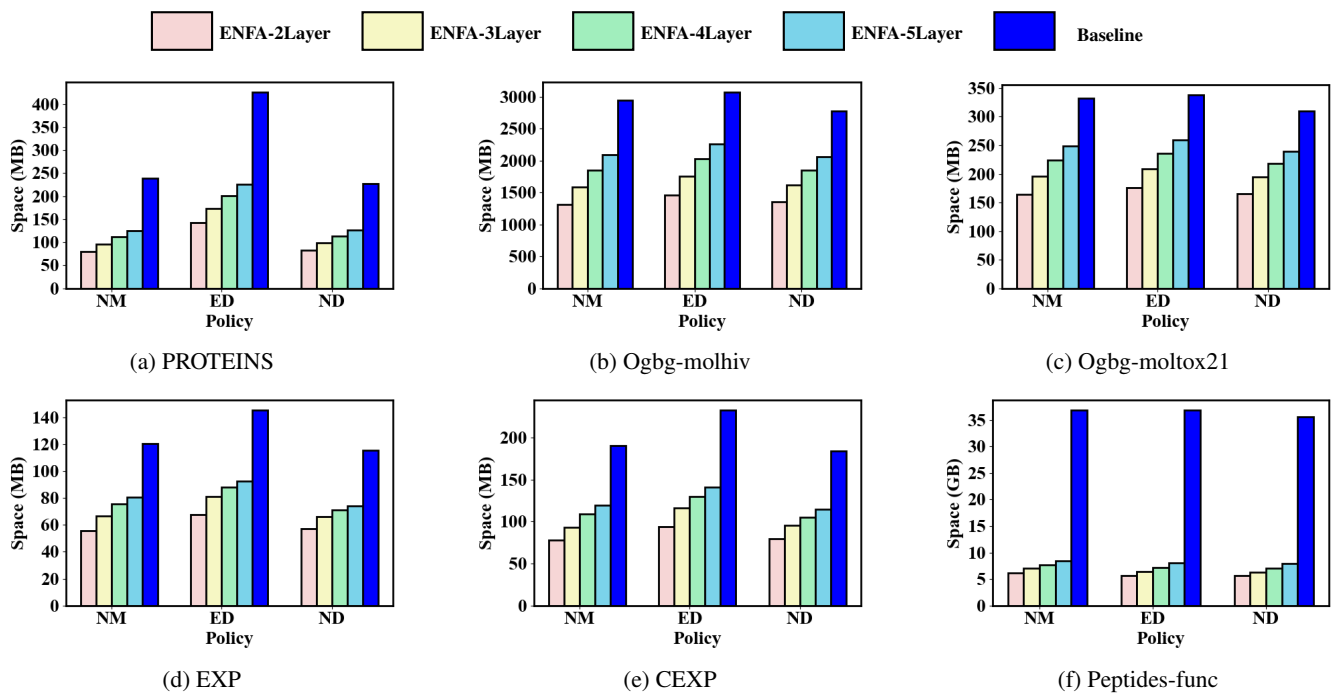
\begin{figure*}[t!]
  \centering
  \begin{center}
  \begin{tikzpicture}
  \filldraw[thick,color=black,fill=bpink](-7.2+\storageFigOffset,-0.2) rectangle(-6.4+\storageFigOffset,0.3);
  \filldraw[thick,color=black,fill=byellow](-4.4+\storageFigOffset,-0.2) rectangle(-3.6+\storageFigOffset,0.3);
  \filldraw[thick,color=black,fill=bgreen](-1.6+\storageFigOffset,-0.2) rectangle(-0.8+\storageFigOffset,0.3);
  \filldraw[thick,color=black,fill=bblue](1.2+\storageFigOffset,-0.2) rectangle(2.0+\storageFigOffset,0.3);
  \filldraw[thick,color=black,fill=blue](4.0+\storageFigOffset,-0.2) rectangle(4.8+\storageFigOffset,0.3);
  \node[text width=0.35cm] at (-7.55+\storageFigOffset,0) {};
  \node[text width=2.0cm] at (-5.3+\storageFigOffset,0) {\scriptsize\textbf{\framework-2Layer}};
  \node[text width=2cm] at (-2.5+\storageFigOffset,0) {\scriptsize\textbf{\framework-3Layer}};
  \node[text width=2cm] at (0.3+\storageFigOffset,0) {\scriptsize\textbf{\framework-4Layer}};
  \node[text width=2cm] at (3.1+\storageFigOffset,0) {\scriptsize\textbf{\framework-5Layer}};
  \node[text width=2cm] at (6.1+\storageFigOffset,0) {\scriptsize\textbf{Baseline}};
  \end{tikzpicture}
  \end{center}
  \vspace{-8pt}
  \subfloat[PROTEINS]{\includegraphics[width=0.3\textwidth]{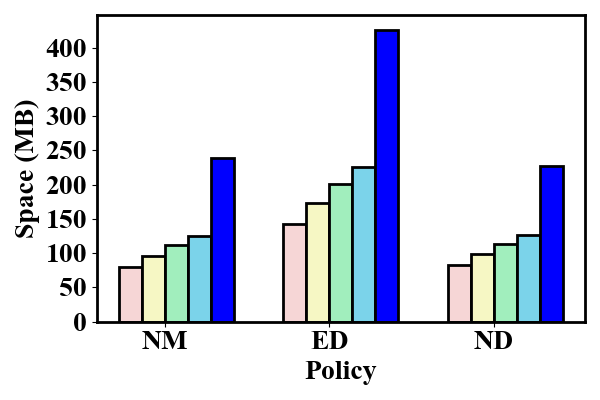}}
  \vspace{-2pt}
  \hfill
  \subfloat[Ogbg-molhiv]{\includegraphics[width=0.3\textwidth]{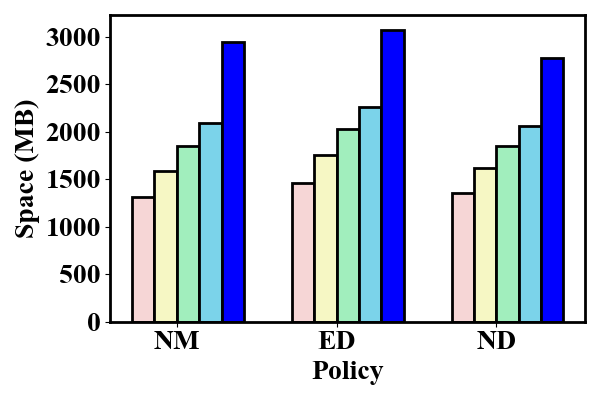}}
  \vspace{-2pt}
  \hfill
  \subfloat[Ogbg-moltox21]{\includegraphics[width=0.3\textwidth]{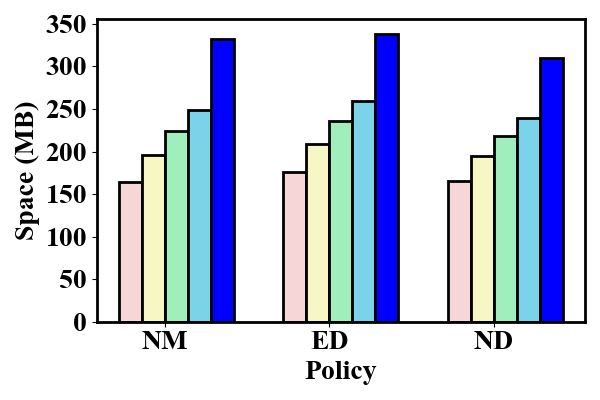}}
  \vspace{-2pt}
  \newline
  \subfloat[EXP]{\includegraphics[width=0.3\textwidth]{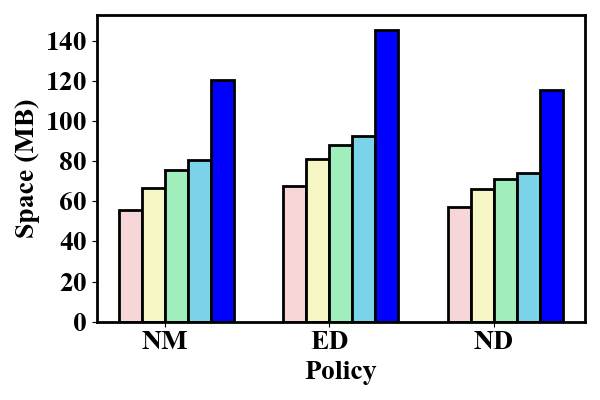}}
  \hfill
  \subfloat[CEXP]{\includegraphics[width=0.3\textwidth]{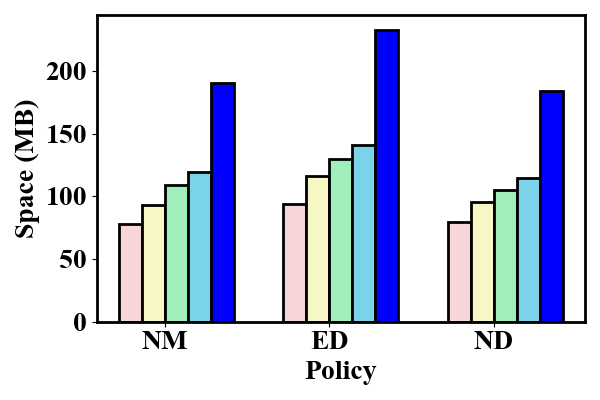}}
  \hfill
  \subfloat[Peptides-func]{\includegraphics[width=0.3\textwidth]{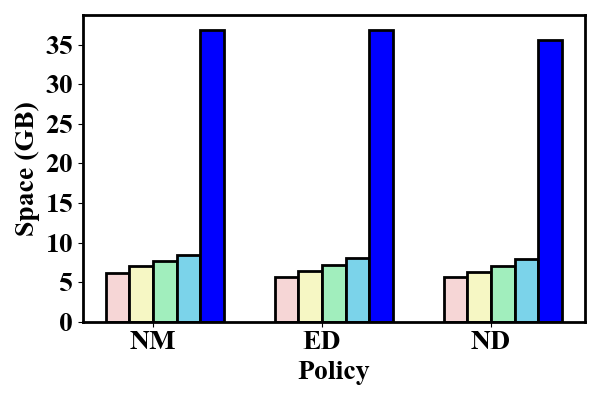}}
  \caption{Experimental Results on Storage Efficiency.}
  \vspace{-4pt}
  \label{fig:exp-eff-stg}
\end{figure*}

\stitle{More Efficient Subgraph GNNs.}
On the other hand, only a few studies have partially focused on the efficiency of subgraph GNNs and they mainly employ sampling strategies in subgraph GNNs.
Bevilacqua \etal~\cite{iclr22esan} propose to randomly select a unified ratio of subgraphs for each graph.
As the number of ordered subgraphs in \cite{nips22osan} is too large, the authors therein propose a data-driven sampling strategy for efficiently training the model.
SUGAR~\cite{www21sugar} suggests employing a reinforcement pooling method to adaptively choose $k$ subgraphs across different epochs.
Kong \etal~\cite{nips23mag} further demonstrate that the reinforcement learning-based sampling strategy could enhance the expressive ability of subgraph GNNs.
Unlike these studies where the sampling strategies often impact performance, our study aims to enhance the efficiency of subgraph GNNs through \textit{exact model design}.
We propose an exact acceleration of subgraph GNNs with full subgraphs and guarantee identical outputs to the subgraph GNNs.
{\framework} is orthogonal to the sampling strategies and suitable for the sampling strategies.

\section{Evaluation}
\label{sec:exp}
\eat{This section presents the evaluation of {\framework}.}
In this section, we aim to addressing the following questions to validate the efficiency of {\framework}.
\textbf{Q1}: How much storage can {\framework} save compared to conventional subgraph GNNs?
\textbf{Q2}: How does {\framework} affect the proportional storage of the data?
\textbf{Q3}: How does the computational efficiency of {\framework} compare to conventional subgraph GNNs?
\textbf{Q4}: How does the computational efficiency of {\framework} compare to subgraph GNNs with subgraph message passings? 

\subsection{Experimental Setup}
\label{subsec:exp-setup}
\stitle{Datasets.}
We evaluate the performance of {\framework} on the following datasets: 
(1) A bioinformatics dataset from TUD repository~\cite{arxiv20tudata}, namely PROTEINS.
Each graph represents a molecule and the task aims to predict certain properties of the given molecules. 
(2) Bioinformatics datasets from OGB~\cite{nips20ogb}, namely ogbg-molhiv and ogbg-moltox21.
(3) Synthetic datasets designed to measure the expressiveness of GNNs~\cite{arxiv20randomnode}, namely EXP and CEXP.
(4) The long range graph dataset, namely Peptides-func, from LRGB~\cite{nips22lrgb}.
All datasets have been processed following previous work~\cite{iclr22esan,nips22sun}.

\begin{figure*}[t!]
  \centering
  \hspace{25pt}
  \subfloat{\includegraphics[width=0.7\textwidth]{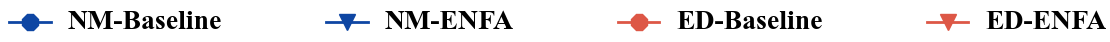}}
  \vspace{-8pt}
  \newline
  \setcounter{subfigure}{0}
  \subfloat[PROTEINS]{\includegraphics[width=0.22\textwidth]{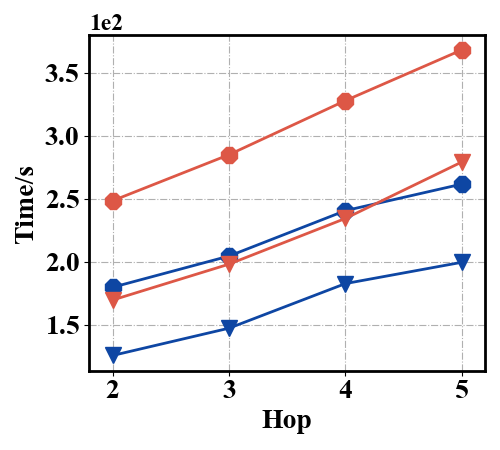}\label{fig:exp-comp-gin-wois-protein}}
  \hfill
  \subfloat[ogbg-molhiv]{\includegraphics[width=0.22\textwidth]{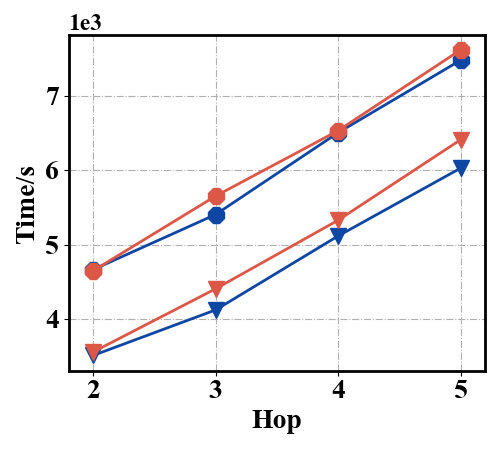}\label{fig:exp-comp-gin-wois-ogbg-molhiv}}
   \hfill
    \subfloat[CEXP]{\includegraphics[width=0.22\textwidth]{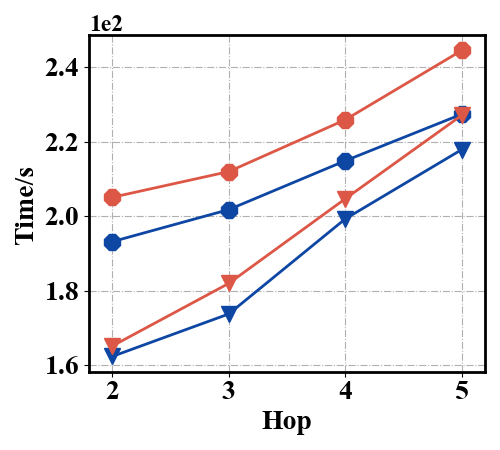}\label{fig:exp-comp-gin-wois-exp}}
  \hfill
  \subfloat[Pep-func]{\includegraphics[width=0.22\textwidth]{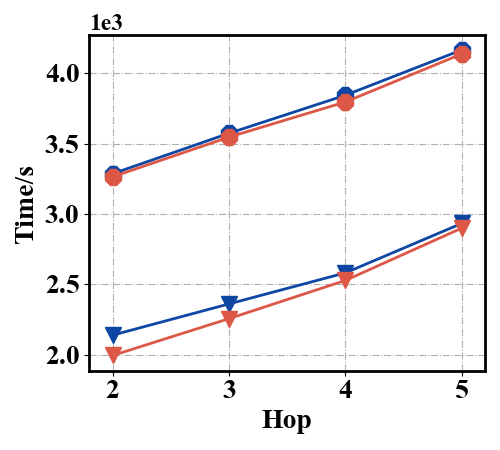}\label{fig:exp-comp-gin-wois-pepfunc}}
  \newline
  \subfloat[PROTEINS with SM]{\includegraphics[width=0.22\textwidth]{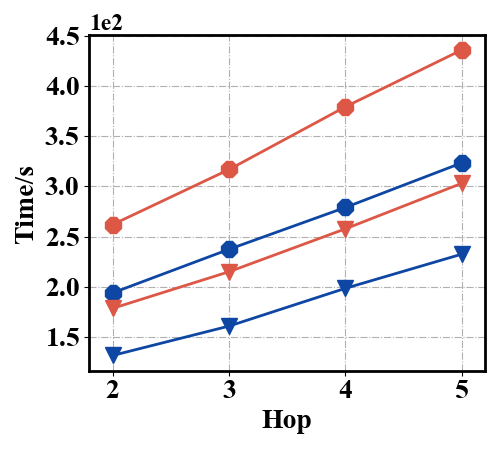}\label{fig:exp-comp-gin-wis-protein}}
  \hfill
  \subfloat[ogbg-molhiv with SM]{\includegraphics[width=0.22\textwidth]{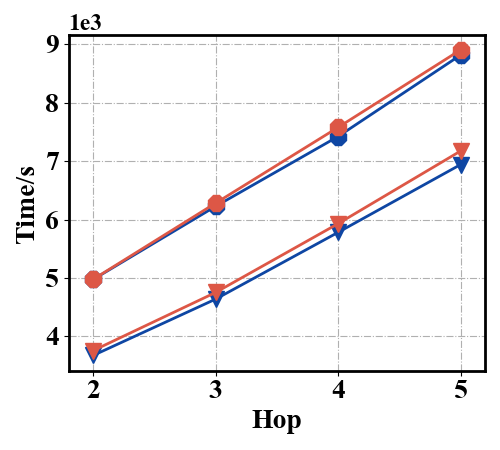}\label{fig:exp-comp-gin-wis-ogbg-molhiv}}
  \hfill
  \subfloat[CEXP with SM]{\includegraphics[width=0.22\textwidth]{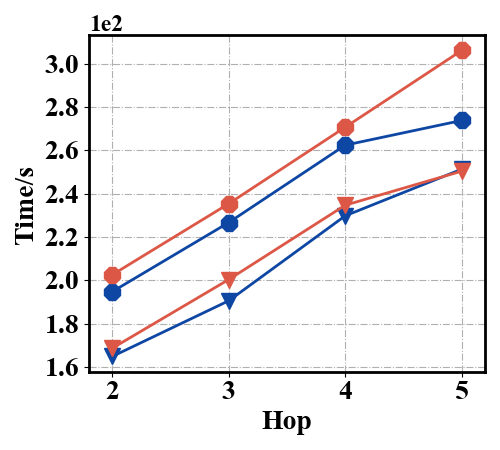}\label{fig:exp-comp-gin-wis-exp}}
  \hfill
  \subfloat[Pep-func with SM]{\includegraphics[width=0.22\textwidth]{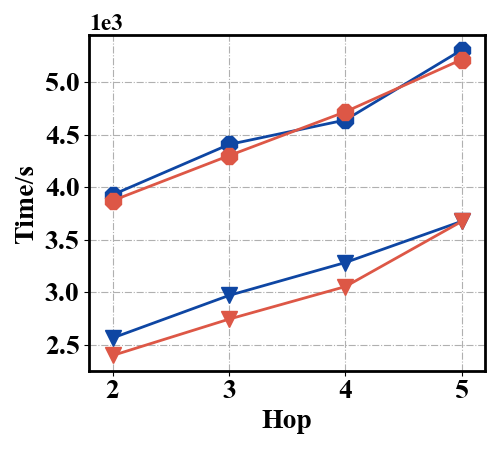}\label{fig:exp-comp-gin-wis-pepfunc}}
  \caption{Experimental Results on Computation Efficiency.}
  \vspace{-10pt}
  \label{fig:exp-computation-gin}
\end{figure*}

\stitle{Benchmark.}
To address \textbf{Q1}, we compare {\framework} with a widely accepted implementation of the conventional subgraph GNN, \ie ESAN~\cite{iclr22esan}.
To thoroughly validate the efficiency of {\framework}, we choose GIN~\cite{iclr18powerful} as MPNN layers, vary the number of layers (ranging from $2$ to $5$), and implement different policies including node-marking (NM), node-deleting (ND), and edge-deleting (ED) for storage efficiency evaluation.
To address \textbf{Q2}, we count the proportion of different types of data in \textbf{Q1} and show the storage breakdown of {\framework} and subgraph GNNs.
To address \textbf{Q3}, we compare {\framework} and ESAN by varying the number of layers and policies as for \textbf{Q1}.
As the computation efficiency of {\framework} mainly stems from the lower memory cost per graph, we select a larger minibatch for {\framework}, ensuring that the size of the input in {\framework} does not exceed that of ESAN, and report the total training time.
To address \textbf{Q4}, we compare {\framework} and the subgraph GNNs with subgraph message passings in \cite{iclr22esan} by varying the number of layers and policies as for \textbf{Q1}.



\subsection{Experimental Results}
We have compared the performance between the baseline and {\framework} in \appref{subsec:append-equivalence}, showing {\framework} can produce very similar performance to subgraph GNNs.
On this basis, we evaluate the storage and computational efficiency of {\framework}.
\subsubsection{Evaluation on Storage Efficiency}

\begin{figure}[t]
\vspace{-2pt}
  \centering
  \subfloat[Node-marking]{\includegraphics[width=0.19\textwidth]{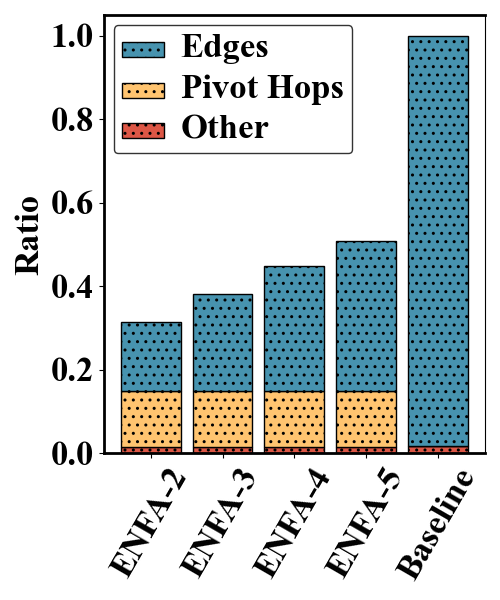}}
  \hspace{4pt}
  \subfloat[Edge-deleting]{\includegraphics[width=0.19\textwidth]{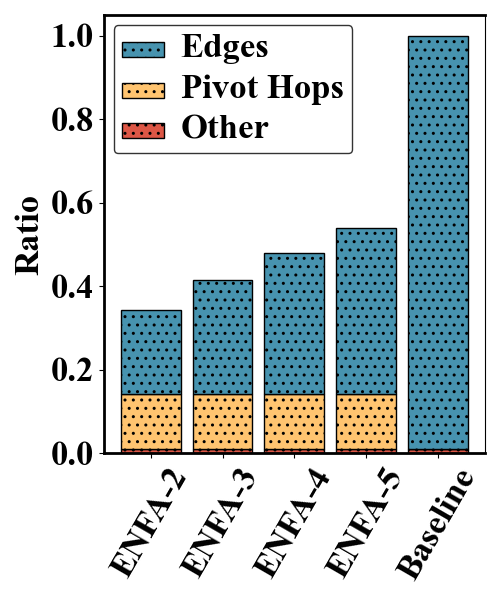}}
  \hfill
  \caption{PROTEINS Storage Breakdown.}
  \vspace{-8pt}
  \label{fig:exp-stg-case}
\end{figure}

Fig.~\ref{fig:exp-eff-stg} shows the experimental results on the storage efficiency of {\framework}.
The storage space of the data file after preprocessing for different policies and varying numbers of layers is reported.
Specifically, the first $4$ bars in each group represent the storage size in {\framework} for $2,3,4,5$ MPNN layers respectively, and the final bar represents the storage size in the baseline.
Overall, {\framework} outperforms the baseline in all cases and saves up to $84.5$\%, $82.6$\%, $80.4$\%, and $78.3$\%\eat{occupies up to $6.5\times$, $5.7\times$, $5.1\times$, and $4.6\times$ less} space compared to subgraph GNNs with $2,3,4$ and $5$ layers,\eat{can save up to 84.5\%, 82.6\%, 80.4\%, and 78.3\% space for subgraph GNNs with $2,3,4$ and $5$ layers} 
respectively.
\eat{
Besides, we observe a larger gap between {\framework} and the baseline in datasets with larger graphs (\eg PROTEINS and peptides-func).
This is because larger graphs yield more subgraphs in subgraph GNNs, and thus {\framework} can save more space compared with the baseline.}
Regarding the results of {\framework} with different parameters, we note that as the number of layers increases, so does the storage space,  as ego nets become larger with increasing layers.

\subsubsection{Evaluation on Storage Breakdown}
We take PROTEINS to examine the breakdown storage of ENFA.
The proportion of different data types in ENFA and the baseline, under NM and ED policies, is illustrated in \figref{fig:exp-stg-case}. 
The blue, yellow, and red portions represent the edges of subgraphs, additional data incurred by ENFA ({\ie} the pivot hops), and other data like features of the original graph, respectively.
The additional information incurred by ENFA does not change as the number of layers increases.
This is because the size of pivot hops\eat{ record the hop for each node in every subgraph, which} is not related to the edges of the ego nets.


\subsubsection{Evaluation on Computation Efficiency}
\figref{fig:exp-comp-gin-wois-protein}-\figref{fig:exp-comp-gin-wois-pepfunc} propose the time cost of {\framework} and the baseline with varying numbers of MPNN layers.
Due to limited space and better readability, we report the results with NM and ED policies on PROTEINS, ogbg-molhiv, EXP, and peptides-func, and the full experimental results can be found in \appref{subsec:append-exp-efficiency}.
{\framework} outperforms the baseline on all datasets and achieves the acceleration by up to $1.66\times$ speedup.
As the number of layers increases, both {\framework} and the baseline require more time for training, but {\framework} still maintains a notable advantage.
Further experiments in \appref{subsec:eval-diff-layers} show that {\framework} has similar improvement on other MPNN layers.

\figref{fig:exp-comp-gin-wis-protein}-\figref{fig:exp-comp-gin-wis-pepfunc} report the partial results on {\framework} and the baseline with subgraph message passings.
The full results can be found in \appref{subsec:append-exp-efficiency}.
We observe a similar trend as seen in \figref{fig:exp-comp-gin-wois-protein}-\figref{fig:exp-comp-gin-wois-pepfunc}.
This is because to fit subgraph message passings, {\framework} only requires slight modification on the internal embeddings $\mathbf{H}_o^{(i)}$ (\ie \equref{eq:enfa-conv-share}), which is not the bottleneck in training models.
\vspace{-2pt}
\section{Conclusions}
This paper concentrates on the exact acceleration of the popular subgraph GNNs.
We observe that a significant number of redundant computations exist, the results of which could be derived from the original graph.
Building on this, we propose {\framework}, a model that ensures identical outputs to subgraph GNNs with complete subgraphs, while only requiring the smaller ego nets as inputs, thereby improving both storage and computational efficiency.
Experimental evaluations demonstrate that {\framework} can considerably reduce storage requirements and enhance computational efficiency.

\bibliography{enfa}

\clearpage
\newpage
\appendix
\newtheorem{apptheorem}{Theorem}

\twocolumn[
    \begin{center}
      \section*{Appendix}
      \vspace{6pt}
    \end{center}
]
\section{Proof of Theorems}
\label{subsec:append-proof}
\subsection{Proof of \thmref{thm:correctness}}
\label{subsec:append-thm1}
We first reclaim \thmref{thm:correctness} here.
\begin{apptheorem}
    The internal embeddings of nodes in $j$-th ego net in {\framework} are identical to those in $j$-th subgraph in the conventional subgraph GNN after $L$ layers.
    Consequently, ENFA can produce the same embeddings $Y$ for $G$ as the conventional subgraph GNN.
\end{apptheorem}

We prove Theorem 1 through an induction on the number of layers, $l$.
From the observation in \figref{fig:repeat-conv}, we have the following base case for $l=1$:
\begin{proposition}
    \label{remark:correctness-init} 
    After the convolution of the first MPNN layer, the internal embeddings $\mathbf{H}^{(1)}$ obtained by ENFA (as per \equref{eq:enfa-conv}) are identical to those in the conventional subgraph GNN (as per \equref{eq:subgraphGNN}).
\end{proposition}

\propref{remark:correctness-init} holds based on the property of the MPNN layers:
the output of a node from an MPNN layer only depends on the input embeddings of \emph{the node itself and its neighbors}.
Consequently, (1) for a node whose pivot hop is no greater than $1$ in $\mathegonet^{L+1}_{j}$, its local structure is the same as that in $G_{j}$, and thus ENFA directly obtains the same embeddings as in the conventional subgraph GNNs; (2) for a node whose pivot hop is greater than $1$, itself and its neighbors have identical features in the subgraph of the conventional subgraph GNN and the original graph, which ensures their internal embeddings replaced by the original graph in ENFA equal those in the conventional subgraph GNN.



Now consider a general case where ENFA has achieved the same internal embeddings as the conventional subgraph GNN after the $(l-1)$-th layers.
\begin{lemma}
\label{lem:correctness-ltolp1}
    Assuming ENFA could produce the same embeddings in the subgraph $\mathegonet^{L+1}_{j}$ as those in $G_{j}$ in the conventional subgraph GNN after $l-1$ MPNN layers, 
    then this equivalence holds after the $l$-th layer for $l\leq L$. 
\end{lemma}
\begin{proof}
    By the assumption, the input embeddings to the $l$-th GNN layer remain the same in ENFA and the conventional subgraph GNN.
    Now consider different nodes based on their pivot hops in the subgraph $\mathegonet_{j}$:
    \begin{itemize}[leftmargin=*]
    \item For a node with pivot hops in $\mathegonet_{j}$ no greater than $l$, its internal embeddings before the $l$-th layer are the same in ENFA and conventional subgraph GNN as per the assumption.
    Their local structures in $\mathegonet^{L+1}_{j}$ and $G_{j}$ are also the same because ENFA generates $(L+1)$-hop ego nets and $l\leq L$.
    Hence, ENFA could generate the same embeddings after the $l$-th layer for these nodes.
    \item Consider the nodes with pivot hops in $G_{j}$ larger than $l$, {\ie} those nodes in $V_{j}^{>l}$.
    According to the definition of pivot hops, the $l$-hop ego nets of these nodes disjoint with $pvt_{j}$, meaning that the internal embeddings of nodes in $V_{j}^{>l}$ after $l$ layers remain the same in $G_{j}$ and the original graph $G$.
    ENFA replaces the internal embeddings of $V_{j}^{>l}$ with the embeddings in the original graph after $l$ layers exactly, thus obtaining the same internal embeddings for those nodes as the conventional subgraph GNN.
    \end{itemize}
In summary, We prove that given the assumption on the $(l-1)$-th layer, ENFA ensures identical outputs to the conventional subgraph GNN after the $l$-th GNN layer.
\end{proof}

Building upon \propref{remark:correctness-init} and \lemref{lem:correctness-ltolp1}, we are now in a position to demonstrate that {\framework} operates identically to the conventional subgraph GNN.
{\framework} 
obtains the same embeddings after $L$ MPNN layers, {\ie} $\mathbf{H}^{(L)}$ in \equref{eq:enfa-conv}, as subgraph GNNs.
As a result, the same holds for $Y$.

\subsection{Proof of \thmref{thm:extension}}
\label{subsec:append-thm2}
We first reclaim the extension of ENFA to subgraph message passing layers and \thmref{thm:extension} here.

\begin{align}
    &\mathbf{H}_o^{(0)}=\mathbf{X}; \quad \overline{E}, \mathbf{H}^{(0)}=\mathcat(\mathesub_{L+1}(pvt)) \label{eq:enfaext-init} \\
    &\left.
    \begin{aligned}
    &\mathbf{H}_o^{(i)}=\mathlayer_i(E, \mathbf{H}_o^{(i-1)})& \\
    &\mathbf{H}^{(i)}=\mathlayer_i(\overline{E}, \mathbf{H}^{(i-1)})  \\
        &\mathbf{H}_j^{(i)}[V_{j}^{>i}] = \mathbf{H}_o^{(i)}[V_{j}^{>i}] (1\leq j \leq q) \\
        &\mathbf{H}_{SM}^{(i)}=\mathcenc_i(\sum\nolimits_{1\leq k\leq q}\mathbf{H}_{k}^{(i)}) \\
&\mathbf{H}_o^{(i)}=\mathbf{H}_o^{(i)}+\mathbf{H}_{SM}^{(i)} \\
& \mathbf{H}_j^{(i)}=\mathbf{H}_j^{(i)}+\mathbf{H}_{SM}^{(i)}(1\leq j\leq q)
    \end{aligned}
    \right\} 1\leq i\leq L \label{eq:enfaext-share} \\
    &Y=\mathpool(\mathbf{H}^{(L)}) \label{eq:enfaext-pool}
\end{align}

\begin{apptheorem}
    \label{thm:extension}
    The internal embeddings of nodes in $j$-th ego net in \equref{eq:enfa-conv-share} are identical to those in $j$-th subgraph in subgraph GNNs with subgraph message passing layers after $L$ layers.
    Consequently, a combination of \equref{eq:enfa-init}, \equref{eq:enfa-conv-share} and \equref{eq:enfa-pool} can produce the same embeddings $Y$ for $G$ as the subgraph GNNs with subgraph message passing layers.
\end{apptheorem}

For better readability, in this proof, the term ``subgraph GNN'' denotes those with subgraph message passing layers.
With additional subgraph message passing layers, the internal embeddings of nodes whose pivot hops larger than $i$ cannot be directly obtained from the original graph, as these layers alter the internal embeddings of such nodes ({\ie} \equref{equ:sgnn-ce}).
Nonetheless, we can still retain the embeddings shared among different subgraphs by adding the embeddings from the subgraph message passing layers to the internal embeddings of the \textit{original graph} after each layer.

\begin{proposition}
\label{prop:add-ce}
Assume a node shares identical internal embeddings in both the subgraph and the original graph before the addition of embeddings from the subgraph message passing layer, then this equivalence also holds after adding the embeddings to the internal embeddings of both subgraphs and the original graph.
\end{proposition}
\begin{proof}
The embeddings from subgraph message passing layers are graph-level embeddings, meaning the embeddings of a node in different subgraphs have the same value.
Thus, the addition of subgraph message passings to both the subgraphs and the original graph does not disrupt the equivalence of a node's embeddings in the subgraphs and the original graph.
\end{proof}


\propref{prop:add-ce} provides some insight into why we add $\mathbf{H}_{SM}^{(i)}$ to both $\mathbf{H}^{(i)}$ and $\mathbf{H}_o^{(i)}$.
We are now ready to prove the theorem through induction.
For the base case $l=1$, the internal embeddings $\mathbf{H}^{(1)}$ are identical in {\framework} and subgraph GNNs before the subgraph message passing layer, that is, before the fourth line of \equref{eq:enfaext-share}, as discussed in \propref{remark:correctness-init}.
Based on \propref{prop:add-ce}, \equref{eq:enfaext-share} could further generate embeddings identical to those of subgraph GNNs after the subgraph message passing layer.

Now consider a general case where {\framework} has achieved identical internal embeddings to those of subgraph GNNs after the $(l-1)$-th layers.

\begin{lemma}
    \label{lem:ext-correctness}
    Assuming \equref{eq:enfaext-init}-\equref{eq:enfaext-pool} could produce the same embeddings in the subgraph $G_j^{L+1}$ as those in $G_j$ in subgraph GNNs and it holds $\mathbf{H}_o^{(l-1)}[V_j^{>l-1}]=\mathbf{H}^{(l-1)}[V_j^{>l-1}]$ after $l-1$ iterations, then this also holds after the $l$-th iteration for $l\leq L$.
\end{lemma}
\begin{proof}
    By the assumption, the input embeddings to the $l$-th MPNN layer remain the same in {\framework} and the original graph\eat{ for the nodes in $V_j^{>l-1}$}.
    Similarly to the discussion in \thmref{thm:correctness}, consider the internal embeddings from the $l$-th MPNN layer \textit{before} the subgraph message passing layer, {\ie} the embeddings after the third line in \equref{eq:enfaext-share}.
    \begin{itemize}
        \item For the nodes with pivot hops in $G_j$ no greater than $l$, their internal embeddings are not replaced.
        Similarly to the discussion in \thmref{thm:correctness}, the structure around these nodes does not change in $G_j^{L+1}$ and $G_j$ and {\framework} could generate the same embeddings $\mathbf{H}^{(i)}$ before the subgraph message passing layer as in subgraph GNNs.
        \item Consider the nodes with pivot hops in $G_j$ larger than $l$, {\ie} $V_j^{>l}$.
        Their neighbors are always in $V_j^{l-1}$.
        Since $V_j^{l-1}$'s embeddings in $\mathbf{H}_o^{(l-1)}$ and $\mathbf{H}^{(l-1)}$ are the same, the embeddings of nodes in $V_j^{>l}$ are identical in $\mathbf{H}_o^{(l)}$ and $\mathbf{H}^{(l)}$ before the subgraph message passing layer.
        Thus, after the placement, {\framework} could generate the same embeddings before the subgraph message passing layer as in subgraph GNNs.
    \end{itemize}
    The above discussion guarantees that {\framework} produces the same internal embeddings as subgraph GNNs and it holds $\mathbf{H}_o^{(l)}[V_j^{>l}]=\mathbf{H}^{(l)}[V_j^{>l}]$ before the subgraph message passing layer in $l$-th iteration.

    Now consider how the subgraph message passing layer affects the internal embeddings.
    $\mathbf{H}^{(l)}$ in both {\framework} and subgraph GNNs add the same $\mathbf{H}_{SM}^{(l)}$, and thus {\framework} could produce the same internal embeddings after the $l$-th subgraph message passing layer.
    Besides, notice that $\mathbf{H}_o^{(l)}$ also adds $\mathbf{H}_{SM}^{(l)}$. Since the embeddings of $V_j^{>l-1}$ in $\mathbf{H}_o^{(l)}$ come from $\mathbf{H}^{(l)}$ and $V_j^{>l}\subset V_j^{>l-1}$, it also guarantees $\mathbf{H}_o^{(l)}[V_j^{>l}]=\mathbf{H}^{(l)}[V_j^{>l}]$ after the subgraph message passing layer.
\end{proof}
    From an induction of \lemref{lem:ext-correctness}, the internal embeddings after $L$ layers in {\framework} are identical to those in subgraph GNNs.
    Thus, {\framework} could generate the same $Y$ as subgraph GNNs.

\section{Discussions on Other Variants}
\label{sec:append-discuss}
\subsection{Subgraph GNNs with Ego-Net Policy}
\label{subsec:append-discuss-egonet}
Bevilacqua~{\etal}~\cite{iclr22esan} suggest using the ego nets around the rooted nodes as subgraphs.
We have two remarks.
(1) Both {\framework} and subgraph GNNs with ego-net policy perform message passing on the ego nets around the rooted nodes.
As such, {\framework} may not exhibit efficiency improvements for the ego-net policy.
However, even though both approaches take ego nets as input, (2) {\framework} essentially differs from subgraph GNNs with the ego net policy.
By identifying and eliminating redundant computations through the convolutions on the original graph and ego nets, {\framework} can produce identical outputs to arbitrary policies, particularly to those using the entire graph as subgraphs, while only needing to perform message passing on small ego nets.  

\subsection{Subgraph GNNs with Edge-Marking Policy}
A recent work~\cite{arxiv23count} suggests adding the structural information of the subgraph, such as degree histogram and distance encodings, among others, to the features of the \textit{rooted edge}. We refer to this approach as the edge-marking policy.
A subgraph GNN adopting an edge-marking policy can be expressed as follows:
\begin{align}
    \label{eq:subgraphGNN-em}
    &\mathsub_{EM}(G)=\{(V,E,\mathbf{X},\mathcat(\mathbf{X}_e,\mathbf{I}_{r,s}))|(v_r,v_s)\in E\} \notag \\
    &\overline{E}, \mathbf{H}^{(0)}, \overline{\mathbf{X}}_e=\mathcat(\mathsub(G)) \notag \\
    &\mathbf{H}^{(i)}=\mathlayer_i(\overline{E}, \mathbf{H}^{(i-1)}, \overline{\mathbf{X}}_e) \text{ for } 1\leq i\leq L \notag \\
    &Y=\mathpool(\mathbf{H}^{(L)})
\end{align}
\noindent 
where $\overline{\mathbf{X}}_e$ refers to the concatenated features of edges from all subgraphs, and $\mathbf{I}_{r,s}$ represents the extra features, where the values for the edge $(v_r,v_s)$ are marked as the structural information, while all other edges are marked with $0$.

In edge-marking policy, only $v_r$ and $v_s$ initially have different features in the subgraph $(V,E,\mathbf{X},\mathcat(\mathbf{X}_e,\mathbf{I}_{r,s}))$ compared to the original graph, thereby making $\{v_r,v_s\}$ the pivot nodes in a manner to the edge-deleting policy.
{\framework} could also produce exact outputs via convolution on ego nets and is anticipated to be both computationally and storage efficient.

\newcommand{\mathind}{\mathtt{Ind}}
\subsection{Subgraph GNNs with Re-Construction Policy}
Previous work~\cite{nips21reconstruction} suggests using $k$-reconstruction of the original graph as the subgraphs.
Formally, let $\mathcal{S}_V^{(k)}$ represent the set of all $k$-size subsets from the vertex set $V$ and let $\mathind_G(S)$ denote the induced edges by a subset of vertex, $S$, from the original graph $G$.
Then the reconstruction policy can be formally defined by $\mathsub_{RC}(G)=\{(S,\mathind_G(S),\mathbf{X}[S])|S\in \mathcal{S}_V^{(k)}\}$, where $\mathbf{X}[S]$ represents the features of nodes in $S$.

For improved readability, we use $n=|V|$ to represent the number of nodes.
When $k=n-1$, the re-construction policy degrades to the node-deleting policy.
Now consider $k>1$.
Each deleted node results in at most $d$ pivot nodes where $d$ is the maximum degree in the graph.
Thus, each subgraph contains at most $(n-k)d$ pivot nodes, and each pivot node results in an $(L+1)$-hop ego net.
The upper bound of the edge size of the subgraphs by {\framework} will be $\binom{n}{k}(n-k)d^{L+2}$.
In comparison, we can obtain the exact size of subgraph edges in the reconstruction policy.
\begin{lemma}
    The size of edges from the subgraphs in $k$-reconstruction policy is $|E|\binom{n-2}{k-2}$.
\end{lemma}
\begin{proof}
    Consider an edge $(v_r,v_s)\in E$.
    We are interested in how many subgraphs the edge exists in.
    In fact, it exists in the subgraphs with node set $\{v_r,v_s\}\cup S_{k-2}$ where $S_{k-2}\in V-\{v_r,v_s\}$ and $|S_{k-2}|=k-2$.
    This indicates that the $(v_r,v_s)$ emerges $\binom{n-2}{k-2}$ times, which is irrelevant to the structures of $v_r$ and $v_s$.
    Thus, the total size of edges from the subgraphs is $|E|\binom{n-2}{k-2}$.
\end{proof}

Thus, we can roughly estimate when {\framework} might employ fewer edges in the subgraphs than the reconstruction policy.
By letting $\binom{n}{k}(n-k)d^{L+2}\leq |E|\binom{n-2}{k-2}$, we find that {\framework} could yield less subgraph edges for graphs satisfying $|E|\geq n(n-1)(n-k)d^{L+2}/k(k-1)$.
As $k$ increases, the value of $(n-k)/k(k-1)$ becomes smaller, which means {\framework} is more suitable for larger $k$, \eg $k$ close to $n$.

\eat{$|E|\geq 2n(n-1)d^{L+2}/(n-2)(n-3)$ when $k=2$ and those satisfying $|E|\geq 3n(n-1)d^{L+2}/(n-3)(n-4)$ when $k=3$.
Moreover, estimating the upper bound of subgraph edges in {\framework} for a general $k$ is non-trivial and is reserved for future research.}

\subsection{Subgraph GNNs with Subsampling Strategies}
Previous studies~\cite{iclr22esan,nips22osan,nips23mag} apply subsampling strategies to subgraphs to enhance the efficiency of subgraph GNNs.
These subsampling strategies reduce the number of subgraphs input to subgraph GNNs, enabling the subgraph GNNs to be trained in an efficient, albeit less accurate, manner. 
A subgraph GNN with subsampling can be formally defined by replacing the first line of Equ. 1 with the following:
\begin{align}
\overline{E},\mathbf{H}^{(0)}=\mathcat(\mathsamp(\mathsub(G))) \notag
\end{align}
where $\mathsamp(\cdot)$ is the subsampling function that selects a subset of subgraphs from $\mathsub(G)$.
Given that our work focuses on reducing the size of each subgraph while maintaining the exact output of subgraph GNNs, {\framework} can be readily applied to subgraph GNNs with subsampling strategies by executing the same subsampling strategy on ego nets within {\framework}.
Formally, {\framework} can precisely imitate the behavior of a subgraph GNN with a subsampling strategy by replacing Equ. 3 with
\begin{align}
    \mathbf{H}_o^{(0)}=\mathbf{X};\text{  } \overline{E},\mathbf{H}^{(0)}=\mathcat(\mathsamp(\mathesub_{L+1}(pvt))). \notag
\end{align}
Correspondingly, as the size of subgraphs is diminished, {\framework} could also decrease the storage space and enhance the efficiency of subgraph GNNs with subsampling strategies.

\begin{figure}
\caption{Example Code of Replacement}
\label{fig:egcode}
\begin{lstlisting}[language=Python]
# h.shape: (pq,f)
# ho,shape: (p,f)
# pvt.shape: (pq)
ind=torch.arange(p).repeat(q)
ho=GNNLayer(o_edge_index, ho)
h=GNNLayer(sub_edge_index, h)
h[pvt>i]=ho[ind[pvt>i]]
# other code
\end{lstlisting}
\end{figure}

\section{Implementation of {\framework}}
{\framework} is developed using PyTorch\footnote{https://pytorch.org} and PyG\footnote{https://github.com/pyg-team/pytorch\_geometric}.
Thanks to the comprehensive tensor operators provided by PyTorch, {\framework} can be implemented efficiently. 
\figref{fig:egcode} presents a sample code for the $i$-th subgraph GNN layer in {\framework}.
Recall that $p$ represents the number of nodes in the original graph and $q$ denotes the number of subgraphs.
$\mathtt{o\_edge\_index}$ and $\mathtt{sub\_edge\_index}$ correspond to the edges of the original graph and subgraphs, respectively.
$\mathtt{ho}$ and $\mathtt{h}$ store the hidden states of the original graph and subgraphs (\ie $\mathbf{H}_o^{(i)}$ and $\mathbf{H}^{(i)}$), respectively.
$\mathtt{pvt}$ represents the pivot hops of each node across all subgraphs.
Notice that $\mathtt{h}$ (respectively $\mathtt{pvt}$) has its first dimension shape $pq$, where the $(s*p+r)$-th entry represents the hidden states (respectively pivot hop) of the $r$-th node in the $s$-th subgraph.

In {\framework}, as illustrated in lines 5-6 in \figref{fig:egcode}, both $\mathtt{h}$ and $\mathtt{ho}$ are input into the GNN layer based on the edges of subgraphs (\ie $\mathtt{sub\_edge\_index}$) and the original graph (\ie $\mathtt{o\_edge\_index}$), respectively.
Then, as shown in line 7, nodes with a pivot hop greater $i$ are filtered by $\mathtt{pvt}>i$.
Note that we use an index tensor $\mathtt{ind}$ with shape $(pq)$ to establish the one-to-many mapping from $\mathtt{ho}$ to $\mathtt{h}$, which avoids extensive duplication of $\mathtt{h}$.
Specifically, the values of $\mathtt{ind}$ are $(0,1,2,...,q,0,1,2,...,q,...)$.
Thus, $\mathtt{ind}[\mathtt{pvt}>i]$ are the indices of those nodes that need to be replaced in the \emph{original graph} and $\mathtt{ho}[\mathtt{ind}[\mathtt{pvt}>i]]$ precisely represents the hidden states from the original graph that are to be replaced.
The employment of $\mathtt{ind}$ is advantageous. Firstly, by employing $\mathtt{ind}$, we avoid the duplication of $\mathtt{h}$ when mapping the hidden states from the original graph to subgraphs, which is space-efficient.
Besides, since $\mathtt{ind}$ is constant, it can be preprocessed before the model training begins.

\section{Supplementary Experiments}
\subsection{Detailed Experimental Settings}
\label{subsec:append-detial-exp-setting}

\begin{table*}[t!]
    \centering
    \small{
    \begin{tabular}{|c|c|c|c|c|c|c|c|}
        \hline
        \multirow{2}{*}{Datasets} & \multirow{2}{*}{Metric}  & \multicolumn{2}{|c|}{node-marking}  & \multicolumn{2}{|c|}{node-deleting} & \multicolumn{2}{|c|}{edge-deleting} \\
        \cline{3-8}
        ~ & ~ &  ESAN  &  ENFA & ESAN & ENFA  & ESAN & ENFA \\
        \hline
        PROTEINS  & Accu. & $76.9\pm 3.9$  & $\mathbf{77.0}\pm 4.8$ &  $77.1\pm 4.6$  & $76.6\pm 4.3$ & $76.8\pm 4.6$  &  $76.6\pm 3.7$  \\
         \hline
         molhiv & ROC  & $76.7\pm 2.4$ & $76.4\pm 2.6$ & $76.2\pm 1.0$ & $\mathbf{76.7}\pm 1.3$ & $76.4\pm 2.1$  &  $\mathbf{76.6}\pm 2.5$\\
         \hline
         moltox21 & ROC  & $75.5\pm 1.2$ & $\mathbf{75.7}\pm 0.8$ & $75.8\pm 0.3$ & $75.2\pm 0.7$ & $75.1\pm 0.5$ & $\mathbf{75.6}\pm 1.0$\\
         \hline  
         EXP & Accu. & $100 \pm 0.0$ & $\mathbf{100}\pm 0.0$ & $100\pm 0.0$ & $\mathbf{100}\pm 0.0$ & $100 \pm 0.0$ & $\mathbf{100}\pm 0.0$ \\
         \hline
         CEXP & Accu. & $100 \pm 0.0$ & $\mathbf{100}\pm 0.0$ & $100 \pm 0.0$ & $\mathbf{100}\pm 0.0$ & $100\pm 0.0$ & $\mathbf{100}\pm 0.0$ \\
         \hline
         Pep-func & AP & $57.2\pm 0.8$ & $56.8 \pm 0.7$ & $57.3 \pm 1.0$ & $\mathbf{57.5} \pm 0.7$ & $57.8 \pm 0.5$ & $57.5 \pm 0.6$ \\
         \hline
    \end{tabular}
    }
    \caption{Best Performance of ENFA and the baseline (\textbf{bold}: {\framework} no worse than the baseline).}
    \label{tab:exp-best-perf}
\end{table*}

\begin{table}[t!]
    \centering
    \begin{tabular}{|c|c|c|c|}
        \hline
        Datasets & \#Graph & Avg. \#Nodes & Avg. \#Edges \\
        \hline
        PROTEINS&   1,113    &   39.1    &       72.8 \\
         \hline
        ogbg-molhiv&41,127   &   25.5    &       27.5  \\
         \hline
        ogbg-moltox21&7,831  &   18.6    &       19.3 \\
         \hline
        EXP     &   1,200    &   44.4    &       55.1  \\
         \hline
        CEXP    &   1,200    &   55.8    &       69.8   \\
         \hline
         peptides-func & 15,535 & 150.9  &  307.3  \\
         \hline
    \end{tabular}
    \caption{Statistics of datasets.}
    \label{tab:dataset-stat}
\end{table}

\textbf{Datasets.}
Statistics for the datasets can be found in \tabref{tab:dataset-stat}. 
In addition to the commonly utilized datasets for testing the performance of GNN models in graph classification, including PROTEINS, ogbg-molhiv, ogbg-moltox21, EXP, and CEXP, we also employ the dataset peptides-func, which contains larger and denser graphs and exhibits long-range dependencies.
The splits of train and test data for the $5$ datasets, bar peptides-func adheres to previous works~\cite{iclr18powerful,iclr22esan}, while we choose the official splits~\cite{nips22lrgb} for peptides-func.

\stitle{Training Setup.}
For conventional subgraph GNNs, the models are trained with a batch size of $32$.
However, for {\framework}, we deliberately choose a larger batch size, ensuring that the maximum data size among all batches for {\framework} does not surpass the maximum data size of a batch while training subgraph GNNs.
The model for the peptides-func dataset is trained for $100$ epochs while the epoch settings for other $5$ datasets follow the guidelines set in \cite{iclr22esan} ({\ie} $350$ for PROTEINS, EXP and CEXP, and $100$ for ogbg-molhiv and ogbg-moltox21).

\stitle{Implementation.}
{\framework} is implemented in PyG.
All experiments are conducted on a machine with $200$GB memory, and an A800 GPU with $80$GB memory.
Each case in the experiment is executed $5$ times and the average is reported.

\subsection{Evaluation on Exact Acceleration}
\label{subsec:append-equivalence}
To validate that {\framework} could produce identical outputs to subgraph GNNs practically, 
we report the best performance of {\framework} and ESAN under identical hyperparameter search strategies.
Except for the peptides-func dataset, we employ a hyperparameter search strategy as utilized in \cite{iclr22esan}.
For peptides-func, we employ $4$ GIN layers as the default setting, choose a learning rate from $\{0.01,0.001\}$, and report the best performance.
This is conducted to determine if {\framework} can match the performance of conventional subgraph GNNs.

\tabref{tab:exp-best-perf} details the optimal performance of ENFA and the baseline under node-marking, edge-deleting, and node-deleting policies.
Overall, ENFA maintains accuracy or ROC that is at least $0.79\%$ lower than the baseline across the $6$ datasets, and it outperforms the baseline in several cases.
Specifically, for the datasets EXP and CEXP, which are designed to measure the expressiveness of the model, we observe that ESAN consistently identifies non-isomorphic graphs as effectively as conventional subgraph GNNs, thereby validating the equivalence.

An interesting observation is that although {\framework} guarantees identical outputs to subgraph GNNs, there is a slight discrepancy in the performance between {\framework} and conventional subgraph GNNs.
This can be attributed to the use of dropout employed in both subgraph GNNs and {\framework}.
Specifically,
after replacing the embeddings of the original graph in the $i$-th layer of {\framework}, the nodes whose embeddings have been replaced, \ie $V_j^{>i}$, will be influenced by the dropout function applied to the \emph{original graph} in the $(i-1)$-th layer.
In contrast, within subgraph GNNs, the embeddings of these nodes will be subjected to dropout independently.


\begin{figure}[t!]
  \centering
  \hspace{10pt}
  \subfloat{\includegraphics[width=0.28\textwidth]{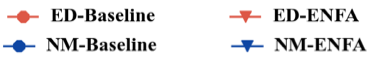}}
  \vspace{-8pt}
  \setcounter{subfigure}{0}
  \subfloat[Ogbg-moltox21]{\includegraphics[width=0.22\textwidth]{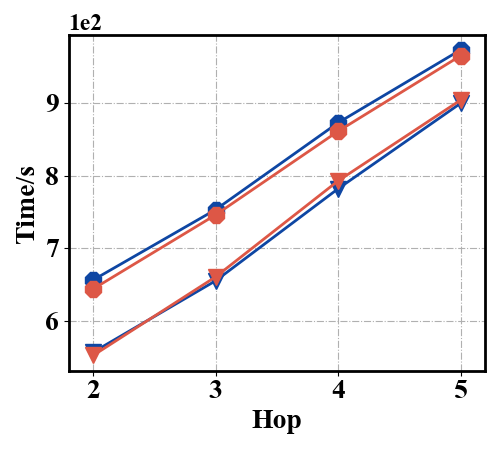}\label{fig:exp-comp-gin-wois-ogbg-moltox}}
   \hfill
  \subfloat[EXP]{\includegraphics[width=0.22\textwidth]{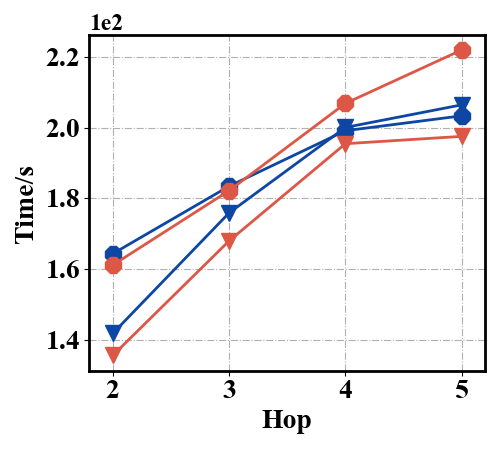}\label{fig:exp-comp-gin-wois-cexp}}
  \newline
  \subfloat[Ogbg-moltox21 w/ SM]{\includegraphics[width=0.22\textwidth]{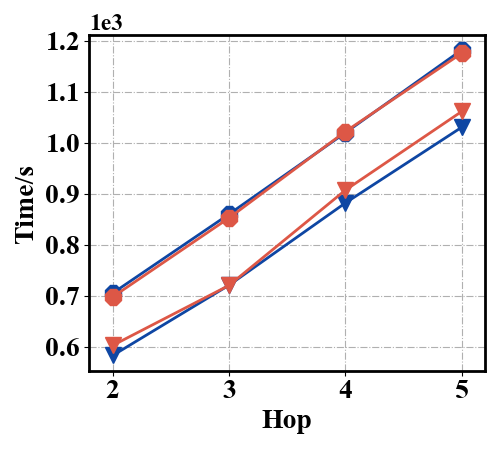}\label{fig:exp-comp-gin-wis-ogbg-moltox}}
  \hfill
  \subfloat[EXP w/ SM]{\includegraphics[width=0.22\textwidth]{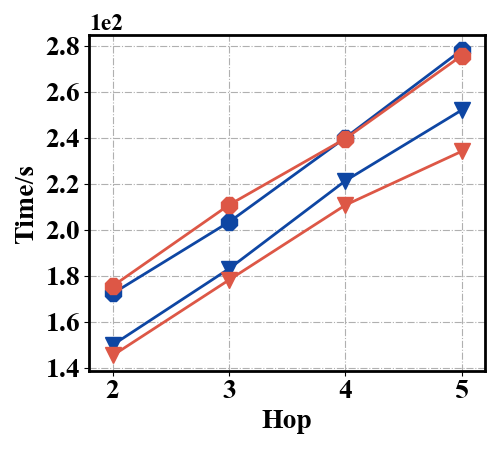}\label{fig:exp-comp-gin-wis-cexp}}
  \caption{Experimental Results on Computation Efficiency (node-marking and edge-deleting policies on datasets ogbg-moltox21 and EXP).}
  \label{fig:exp-computation-app}
\end{figure}

\begin{figure}[t!]
  \centering
  \hspace{10pt}
  \subfloat{\includegraphics[width=0.28\textwidth]{figures/EXP/EFF-Time/EFF-Computation-legend-double.png}}
  \vspace{-10pt}
  \setcounter{subfigure}{0}
  \newline
  \subfloat[PROTEINS (GraphConv)]{\includegraphics[width=0.22\textwidth]{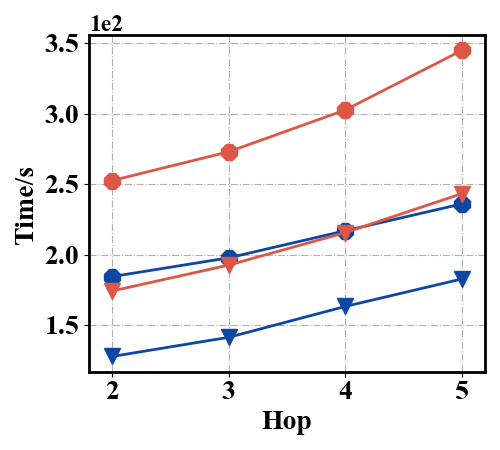}\label{fig:exp-comp-gc-wois-protein}}
  \hfill
  \subfloat[EXP (GraphConv)]{\includegraphics[width=0.22\textwidth]{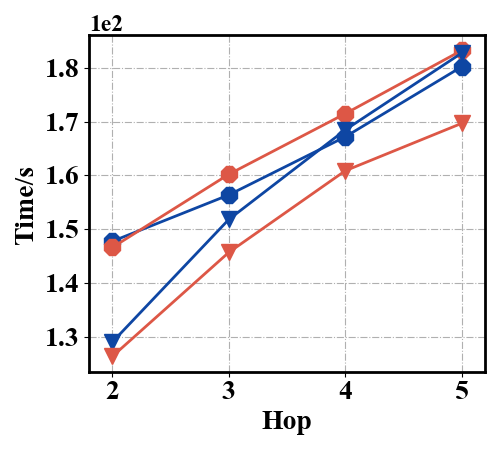}\label{fig:exp-comp-gc-wois-exp}}
   \newline
    \subfloat[CEXP (GraphConv)]{\includegraphics[width=0.22\textwidth]{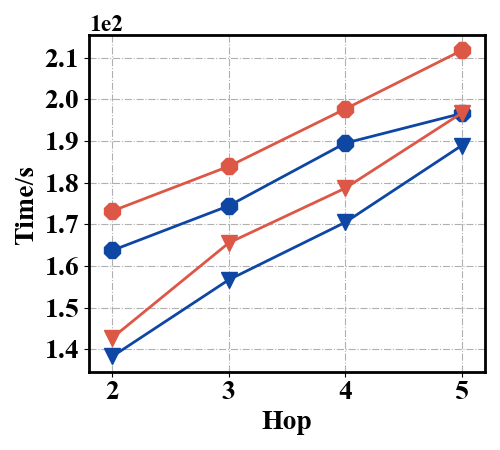}\label{fig:exp-comp-gc-wois-exp}}
    \hfill
   \subfloat[ogbg-molhiv (GCN)]{\includegraphics[width=0.22\textwidth]{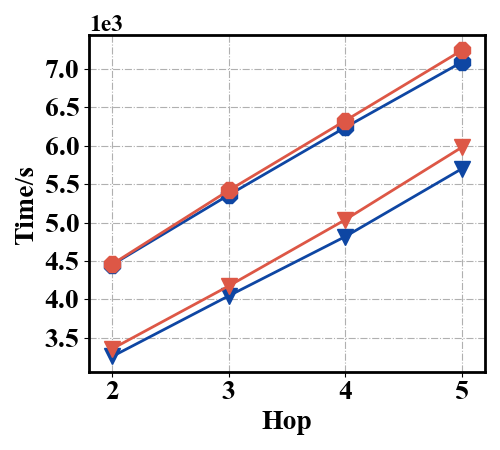}\label{fig:exp-comp-gc-wois-ogbg-molhiv}}
   \newline
   \subfloat[ogbg-moltox21 (GCN)]{\includegraphics[width=0.22\textwidth]{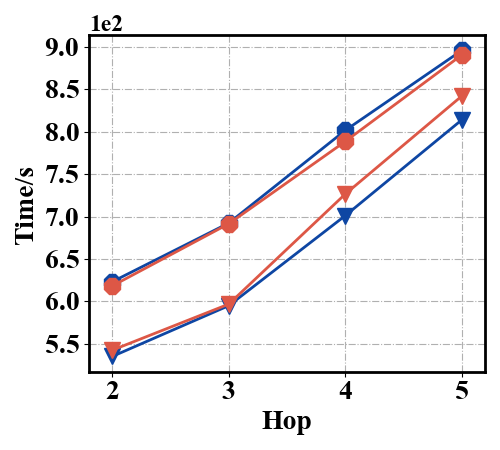}\label{fig:exp-comp-gc-wois-ogbg-moltox}}
   \hfill
  \subfloat[Pep-func (GCN)]{\includegraphics[width=0.22\textwidth]{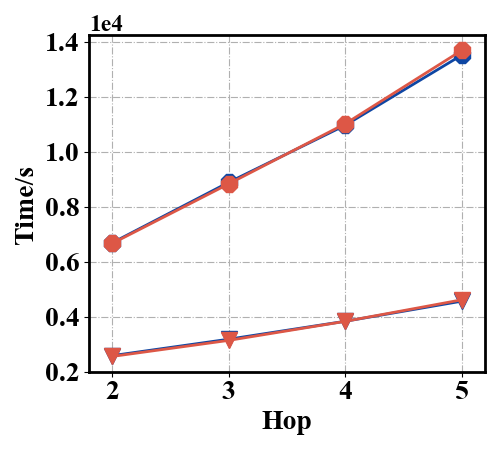}\label{fig:exp-comp-gc-wois-cexp}}
  
  \caption{Experimental Results on Computation Efficiency.}
  \vspace{-10pt}
  \label{fig:exp-computation-gc-wis}
\end{figure}


\subsection{Full Evaluation on Computational Efficiency}
\label{subsec:append-exp-efficiency}

\figref{fig:exp-computation-app} provides the supplementary experiments on the computational efficiency of {\framework} on ogbg-moltox21 and EXP using node-marking and edge-deleting policies.
For conventional subgraph GNNs (\figref{fig:exp-comp-gin-wois-ogbg-moltox} and \figref{fig:exp-comp-gin-wois-cexp}), {\framework} outperforms the baseline in all cases except when employing the node-marking policy on EXP with $4$ and $5$ MPNN layers.
For the subgraph GNNs incorporating subgraph message passings (\figref{fig:exp-comp-gin-wis-ogbg-moltox} and \figref{fig:exp-comp-gin-wis-cexp}), {\framework} outperforms the baseline in all cases.
The integration of additional subgraph message passing layers in the training process leads to a longer training time for both {\framework} and the baseline (\figref{fig:exp-comp-gin-wois-ogbg-moltox} vs. \figref{fig:exp-comp-gin-wis-ogbg-moltox}, \figref{fig:exp-comp-gin-wois-cexp} vs. \figref{fig:exp-comp-gin-wis-cexp}).
However, {\framework} still shows significant time cost improvement compared to the baseline.
This observation is in line with our complexity analysis, which identifies the bottleneck for training models as the stacked MPNN layers on extensive subgraphs. {\framework} continues to exhibit advantages when considering subgraph message passing layers.

The first ({\resp} last) two columns of \figref{fig:exp-computation-gin-nd} report the experimental results on computation efficiency under the \emph{node-deleting policy} without ({\resp} with) the subgraph message passings.
Once again, {\framework} outperforms the baseline in most cases.
Besides, we note that the experimental results under the node-deleting policy are more akin to the node-marking policy than the edge-deleting policy, especially on PROTEINS, EXP, and CEXP.
This is because both node-marking and node-deleting policies generate an equal number of subgraphs for each graph instance (precisely, $|V|$), thereby yielding a similar size of subgraphs for training.

\begin{figure*}[t!]
  \centering
  \hspace{50pt}
  \subfloat{\includegraphics[width=0.34\textwidth]
  {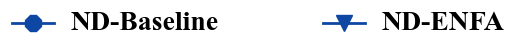}}
  \vspace{-5pt}
  \setcounter{subfigure}{0}
  \newline
  \subfloat[PROTEINS]{\includegraphics[width=0.25\textwidth]{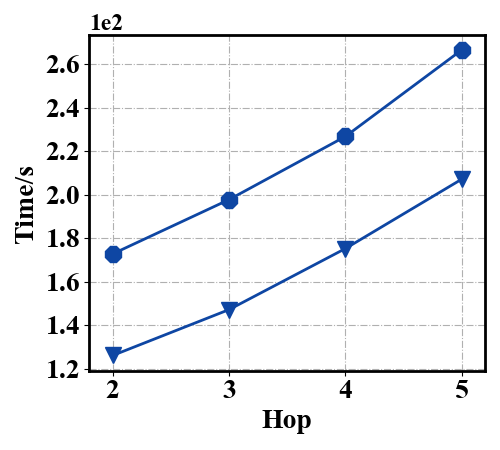}\label{fig:exp-comp-gin-wois-proteins-nd}}
  \hfill
  \subfloat[ogbg-molhiv]{\includegraphics[width=0.25\textwidth]{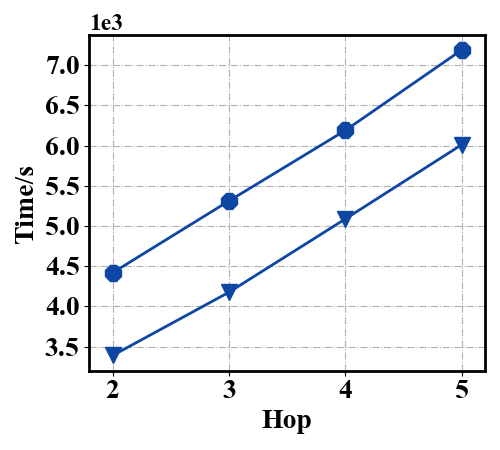}\label{fig:exp-comp-gin-wois-ogbg-molhiv-nd}}
   \hfill
    \subfloat[PROTEINS w/ SM]{\includegraphics[width=0.25\textwidth]{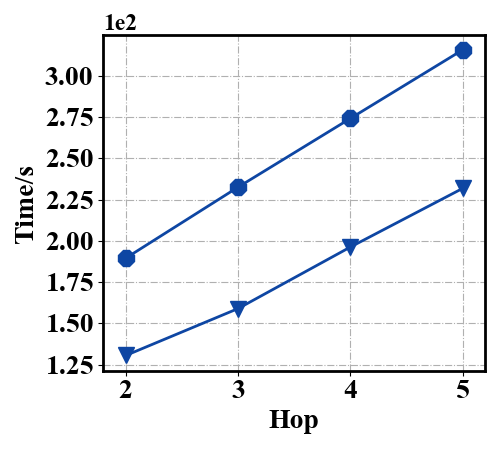}\label{fig:exp-comp-gin-wis-proteins-nd}}
  \hfill
  \subfloat[Ogbg-molhiv w/ SM]{\includegraphics[width=0.25\textwidth]{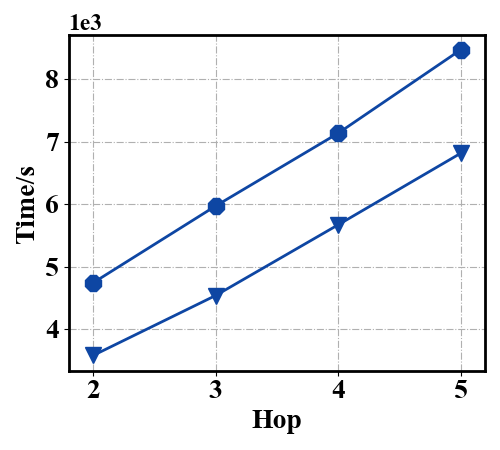}\label{fig:exp-comp-gin-wis-ogbg-molhiv-nd}}
  \hfill
  \newline
  \subfloat[ogbg-moltox21]{\includegraphics[width=0.25\textwidth]{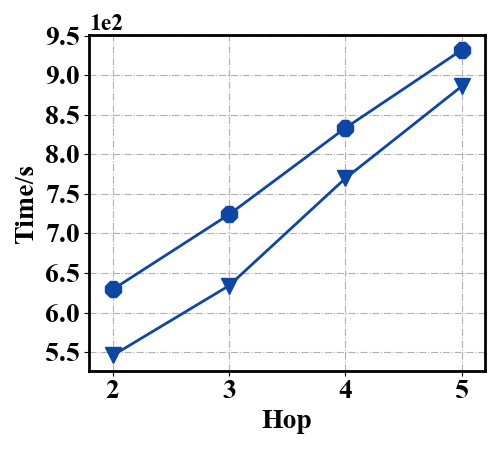}\label{fig:exp-comp-gin-wois-ogbg-moltox-nd}}
  \subfloat[EXP]{\includegraphics[width=0.25\textwidth]{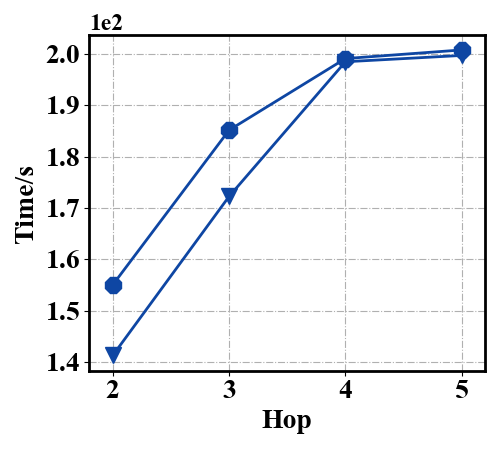}\label{fig:exp-comp-gin-wois-exp-nd}}
  \hfill
  \subfloat[Ogbg-moltox21 w/ SM]{\includegraphics[width=0.25\textwidth]{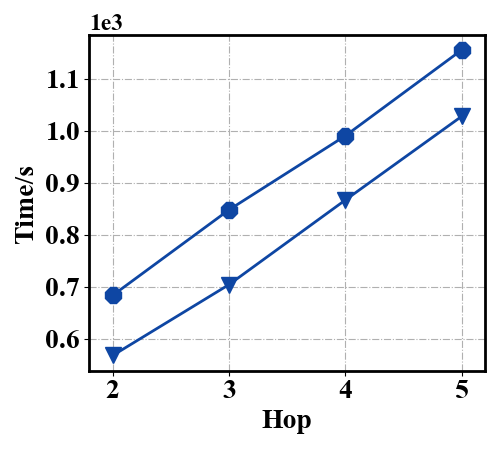}\label{fig:exp-comp-gin-wis-ogbg-moltox-nd}}
  \hfill
  \subfloat[EXP w/ SM]{\includegraphics[width=0.25\textwidth]{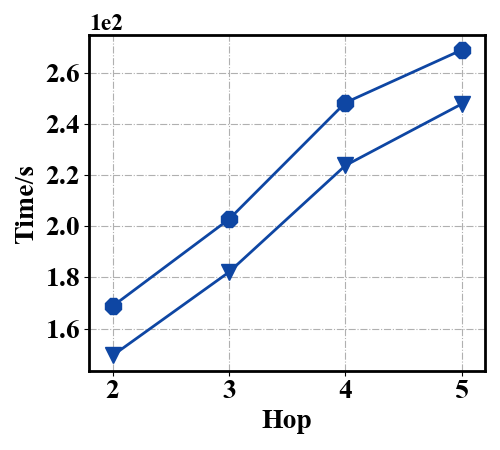}\label{fig:exp-comp-gin-wis-exp-nd}}
  \newline
  \subfloat[CEXP]{\includegraphics[width=0.25\textwidth]{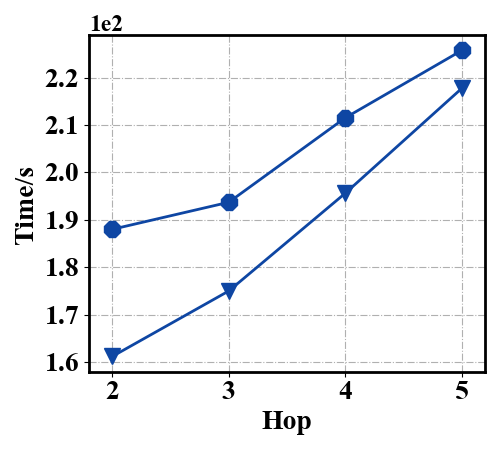}\label{fig:exp-comp-gin-wois-cexp-nd}}
  \subfloat[Pep-func]{\includegraphics[width=0.25\textwidth]{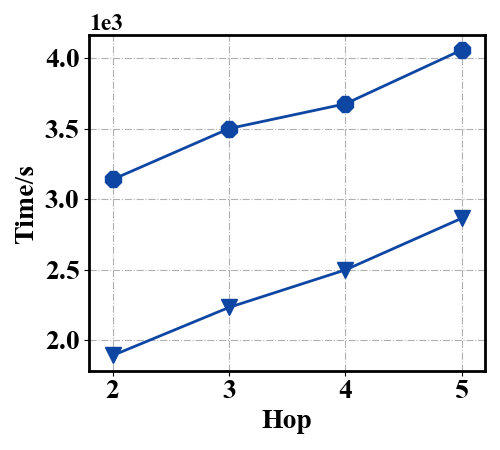}\label{fig:exp-comp-gin-wois-pep-func-nd}}
  \hfill
  \subfloat[CEXP w/ SM]{\includegraphics[width=0.25\textwidth]{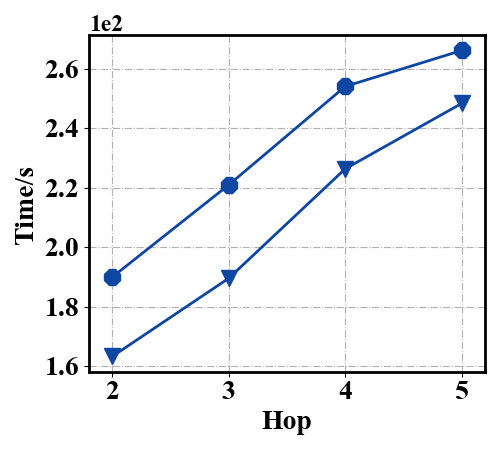}\label{fig:exp-comp-gin-wis-cexp-nd}}
  \hfill
  \subfloat[Pep-func w/ SM]{\includegraphics[width=0.25\textwidth]{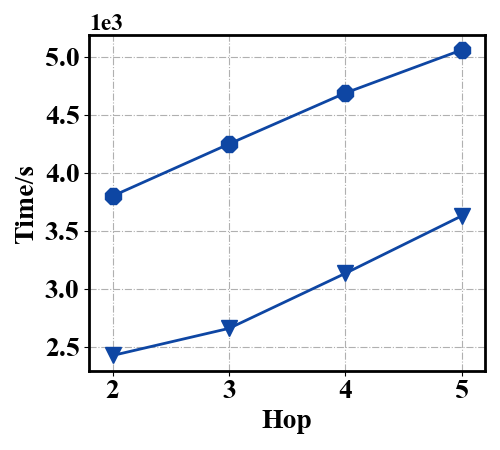}\label{fig:exp-comp-gin-wis-peptides-func-nd}}
  \caption{Experimental Results on Computation Efficiency with node-deleting policy.}
  \label{fig:exp-computation-gin-nd}
\end{figure*}

\subsection{Evaluation on Different MPNN Layers}
\label{subsec:eval-diff-layers}
We further compare {\framework} with subgraph GNNs under different types of MPNN layers to validate the effectiveness of {\framework}.
Following \cite{iclr22esan}, we assess the running time of {\framework} and the baseline with GraphConv layers~\cite{morris2019weisfeiler} for datasets PROTEINS, EXP, and CEXP, and GCN layers~\cite{arxiv16semi} for datasets ogbg-molhiv, ogbg-moltox21, and peptides-func.
\figref{fig:exp-computation-gc-wis} reports the experimental results.

Overall, we observe a trend similar to the experimental results with GIN layers.
In most cases, {\framework} outperforms the baseline, except for the models with $4$-$5$ MPNN layers under the node-marking policy on the EXP dataset.
The time cost difference between the node-marking and edge-deleting policies is more marginal on PROTEINS and CEXP than the other $4$ datasets.
When comparing different MPNN layers (\figref{fig:exp-computation-gc-wis} vs. \figref{fig:exp-comp-gin-wois-protein}-\figref{fig:exp-comp-gin-wois-pepfunc} and \figref{fig:exp-comp-gin-wois-ogbg-moltox}-\figref{fig:exp-comp-gin-wois-cexp}), models with GraphConv and GCN layers require slightly less training time compared to those with GIN layers for the $5$ datasets, except for peptides-func.
However, they are $1.3\times$-$3.3\times$ slower for peptides-func.
This is attributable to the design of GCN layers, which entails more operations to acquire the adjacency matrix scaled by both source nodes and destination nodes.
The graph size of peptides-func is much larger, thereby increasing the time cost with GCN layers.

\end{document}